\documentclass[pmlr]{jmlr}% new name PMLR (Proceedings of Machine Learning Research)
\usepackage{todonotes} 

\usepackage{longtable}% for long tables

\usepackage{xfrac}
\usepackage{siunitx}
\newcommand\mc{\multicolumn{1}{c}}

\usepackage{ulem} % à enlever
\usepackage{multirow}

\usepackage{booktabs}
\theorembodyfont{\upshape}
\theoremheaderfont{\scshape}
\theorempostheader{:}
\theoremsep{\newline}

\jmlrvolume{204}
\jmlryear{2023}
\jmlrworkshop{Conformal and Probabilistic Prediction with Applications}
\jmlrproceedings{PMLR}{Proceedings of Machine Learning Research}
\title[Data-driven Reachability using Christoffel Functions]{Data-driven Reachability using Christoffel Functions and Conformal Prediction}
\author{\Name{Abdelmouaiz Tebjou} \Email{abdelmouaiz.tebjou@irt-systemx.fr}\\
\addr IRT SystemX, 2 boulevard Thomas Gobert, 91120 Palaiseau, France.\\
\addr U2IS, ENSTA Paris, Institut Polytechnique de Paris, Palaiseau, France.
\AND
\Name{Goran Frehse} 
\Email{goran.frehse@ensta-paris.fr}\\
\addr U2IS, ENSTA Paris, Institut Polytechnique de Paris, Palaiseau, France.%
\AND
\Name{Fa\"icel Chamroukhi}
\Email{faicel.chamroukhi@irt-systemx.fr}\\
\addr IRT SystemX, 2 boulevard Thomas Gobert, 91120 Palaiseau, France.
}
\editor{Harris Papadopoulos, Khuong An Nguyen, Henrik Boström and Lars Carlsson}

\begin{document}

\maketitle

\begin{abstract}
An important mathematical tool in the analysis of dynamical systems is the approximation of the reach set, i.e., the set of states reachable after a given time from a given initial state. This set is difficult to compute for complex systems even if the system dynamics are known and given by a system of ordinary differential equations with known coefficients. In practice, parameters are often unknown and mathematical models difficult to obtain. Data-based approaches are promised to avoid these difficulties by estimating the reach set based on a sample of states. If a model is available, this training set can be obtained through numerical simulation. In the absence of a model, real-life observations can be used instead. 
A recently proposed approach for data-based reach set approximation uses Christoffel functions to approximate the reach set. 
%Under the assumption of IID samples, the approximation is guaranteed to converge to the true solution. In this paper, we improve the approach by using statistical guarantees from conformal prediction with a training and calibration set, which substantially improves the sample efficiency and relies on the weaker assumption of exchangeable samples. 
Under certain assumptions, the approximation is guaranteed to converge to the true solution. In this paper, we improve upon these results by notably improving the sample efficiency and relaxing some of the assumptions by exploiting statistical guarantees from conformal prediction with training and calibration sets. 
In addition, we exploit an incremental way to compute the Christoffel function to avoid the calibration set while maintaining the statistical convergence guarantees.
Furthermore, our approach is robust to outliers in the training and calibration set.
\end{abstract}

\section{Introduction}
\label{sec:intro}

%An important mathematical tool in analyzing dynamical systems is the approximation of the reach set, i.e., the set of states reachable after a given time from a given initial configuration. This set is difficult to compute for complex systems even if the system dynamics are known and given by a system of ordinary differential equations with known coefficients. In practice, parameters are often unknown and mathematical models difficult to obtain. Data-based approaches avoid these difficulties by estimating the reach set based on a sample of states. If a model is available, this training set can be obtained through numerical simulation. In the absence of a model, real-life observations can be used instead. 
%A recently proposed approach for data-based reach set approximation uses Christoffel functions to approximate the reach set. Under the assumption of i.i.d. samples, the approximation is guaranteed to converge to the true solution. In this paper, we improve the approach by using statistical guarantees from conformal prediction with a training and calibration set, which substantially improves the sample efficiency and relies on the weaker assumption of exchangeable samples. In addition, we exploit an incremental way to compute the Christoffel function to avoid the calibration set while maintaining the statistical convergence guarantees. Furthermore, our approach is robust to outliers in the training and calibration set.

The problem of reach set approximation arises in different branches of applied mathematics and computer science, and in particular in control theory. In mathematics, the study of initial value problems and their guaranteed solution raises the question of which states can be reached under different configurations; see, for instance the work of \cite{BerzM98}. In computer science, the computation of reach sets is a fundamental operation in formal methods, which establish the correctness of a system with mathematical rigor. Initially, it was applied to program analysis, e.g., by \cite{halbwachs94b}. Later, the approach was extended to cyber-physical systems, which can involve interacting physical components, software, and communication channels, see \cite{alur2015principles}.
Reach set approximations may take different forms based on whether the focus is on scalability, tightness, or efficient computability. Examples include polyhedra, ellipsoids, polynomial zonotopes, and others; see the overview by \cite{althoff2021set}.
In this paper, we establish reach set approximations that are sublevel sets of polynomials, more precisely, sum-of-squares (SOS) polynomials, which are computationally advantageous. Once established, these can readily be used to investigate properties of regions of attraction, stability, and safety or to solve optimization problems. To achieve this, polynomial reach set approximations have been used as barrier certificates, inductive invariants, or Lyapunov functions; see the survey by \cite{doyen2018handbook}.

Traditionally, reach set approximations are established from first principles, starting from a mathematical model of the dynamics. This approach is limited to cases where sufficiently simple models are available and precise enough. More recently, data-based approaches have been used to deal with systems whose dynamics are too complex or where a model is not available and only observations are at hand.
In the following, we provide a brief overview of such approaches.

\paragraph{Related Work}
The traditional approach to go from data to reach set approximations is to first identify a model of the system dynamics and then analyse the model. To give an example, a linear model can be identified efficiently by subspace identification as proposed by \cite{van2012subspace} and then one of the set-based techniques in the survey by \cite{althoff2021set} can be applied to approximate the reach set at a given time in the future.
This can be extended to uncertain linear models and nonlinear systems based on linearization, as pursued by \cite{Alanwar_2023}.
More recently, it has been proposed to derive reach set approximations more directly from data, e.g., the approach of \cite{10.1145/3447928.3457355} uses Taylor series expansions and Lipschitz bounds to derive reach sets for nonlinear systems.
These approaches can, in principle, bound the reach set over an arbitrary time horizon, but the approximation error may increase very rapidly with time. Furthermore, these approaches struggle with complex dynamics.

Our goal in this paper is different and more modest: We establish an SOS polynomial whose sublevel set contains the reachable set in the sense of a \emph{probably approximately correct} (PAC) property. In particular, we consider the approximation of a single time step. This is sufficient for many of the applications considered above (as a first step in constructing barrier certificates, inductive invariants etc.), but in contrast to the approaches cited in the beginning of this section, it does not readily extend to extrapolating the reach set over longer time horizons (it would involve costly quantifier elimination).

One of the earliest data-driven approaches involving SOS polynomials was the construction of barrier certificates by \cite{prajna2006barrier}, e.g., to show that obstacles are avoided by a control system. The scalability was later improved by \cite{7402508}, but the optimisation problem remains somewhat challenging.
Approximating the reach set is related to approximating the support of a probability measure, as observed by \cite{devonport}. Recent work by \cite{lasserre2019empirical,CRMATH_2022__360_G8_919_0} suggests that Christoffel functions are particularly useful for approximating the support. 
Our work is heavily inspired by \cite{devonport}, who proposed to approximate the one-step reach set with an SOS polynomial that is the superlevel set of the Christoffel function.
The PAC guarantees provided by \cite{devonport} are derived from measure theory and are, in practise, somewhat conservative. Based on conformal prediction, we propose significant improvements that we outline below.
Further work on conformal prediction will be cited in the text. 

\paragraph{Contributions}
In this paper, we make the following contributions:
\begin{itemize}
\itemsep=0ex
\item 
We use conformal prediction to provide stronger and more sample-efficient guarantees on reach set approximation than those given by \cite{devonport}.
\item We propose a version of reach set approximation that is robust to outliers, in contrast to the approach of \cite{devonport}. 
\item We exploit an incremental form of the Christoffel function for transductive conformal prediction, thanks to which we don't need to split the data set into training and calibration sets.
\item To the best of our knowledge, this is the first use of the Christoffel function in conformal prediction. The particular properties of the Christoffel function in set and density approximation make it an excellent candidate for a nonconformity function.
\end{itemize}

\paragraph{Structure of the paper}
The %remainder of this 
paper is organized as follows. 
Section \ref{sec:Data-Driven SA with Christoffel Functions} presents the data-driven framework for reachability analysis using Christoffel functions. It describes  the theoretical developments related to the reach set approximation and to Christoffel functions. 
%Section \ref{ssec:Reach-Set-Approximation}
%Section \ref{ssec:Set Approximation with Christoffel Functions}
In Section \ref{Reach Set Approximation with Conformal Prediction}, we introduce our proposed approach to the 
reach set approximation with conformal prediction, whose statistical guarantees are presented in Section \ref{ssec:Statistical Guarantees}. Section \ref{ssec:Avoiding the Calibration Set} presents a technique to avoid the calibration set by using transductive conformal prediction and an incremental version of the Christoffel function. In Section \ref{sec:Robustness to Outliers}, we discuss the robustness of our methodology to outliers. 
Section \ref{sec:Experiments} provides numerical experiments on simulated data to support our theoretical results, and to highlight the effectiveness and potential of the proposed approach. 

\section{Data-driven Reach Set Approximation with Christoffel Functions}
\label{sec:Data-Driven SA with Christoffel Functions}

%\subsection{Reach Set Approximation}
\label{ssec:Reach-Set-Approximation}

%introduire les couples ici
Reachability analysis aims to determine the possible future states of a dynamical system starting from a given initial state. For our purposes, we consider the system to be defined (explicitly or implicitly) by a transition function 
which maps a state $\Vec{x} \in \mathbb{R}^{n}$ to its successor state.
We forego extending the notation to nondeterministic or stochastic systems, since our focus is on estimating the image of $f$ applied to a set of initial states; in the case of a stochastic system we are interested in approximating the support of the image distribution.
%
% The problem we want to solve is described as follows :
Beginning with a given initial set of states $\set{I}$, we are interested in computing the \emph{reachable set}
% $$\set{S}= \{ f(\Vec{x}, \Vec{u}) : \Vec{x} \in \set{I} \}$$
$$\set{S}= \{ f(\Vec{x}) : \Vec{x} \in \set{I} \}.$$
 When $f$ is not precisely known or complex, obtaining the exact solution may not be possible or economical. Instead, we compute an approximation $\hat{\set{S}}$ that covers most of $\set{S}$. 
 Every set $S$ can be represented by a probability measure $\mu$ such as $S$ is the support of $\mu$. This motivated \cite{devonport} to use the Christoffel function to approximate the set $\set{S}$. 
 In the following subsection, we introduce the Christoffel function, its empirical counterpart, and discuss how to compute it.

\subsection{Preliminaries}
We start by introducing some mathematical notation. Given a vector $\Vec{x} \in \mathbb{R}^n$, we denote its elements as  
 $\Vec{x} = ({x}_1,..., {x}_n) $. An integer coefficient vector $\Vec{\alpha} = ({\alpha}_1,...,{\alpha}_n ) \in \mathbb{N}^n$ defines the monomial $\Vec{x}^{\Vec{\alpha}} = {x}_1^{{\alpha}_1}\times {x}_2^{{\alpha}_2}...\times {x}_n^{{\alpha}_n} $.
% \\
For $d \in \mathbb{N}$, we consider $\mathbb{R}[\Vec{X}]_d^n$ to be the vector space of n-variate polynomials whose degree is less or equal to $d$. 
With each coefficient vector $\Vec{\alpha} \in \mathbb{N}^n$, we associate the monomial $\Vec{x}^{\vec{\alpha}}$ whose degree is equal to $\| \Vec{\alpha} \|=\sum_{i=1}^{n}\Vec{\alpha}_i$. 
The monomials $\Vec{x}^{\Vec{\alpha}}$ with $\| \Vec{\alpha} \| \leq d $ form a canonical basis of $\mathbb{R}[\Vec{X}]_d^n$. 
We denote the number of monomials of degree less or equal to $d$ with $$s(d) = \binom{n+d}{n}.$$ 
Let $\Vec{v}_d(\Vec{x}) \in \mathbb{R}^{s(d)}$ be the vector of monomials of degree less or equal to d evaluated at $\Vec{x}$. 
For example, if $d= 2$ and $n = 2$, then $\Vec{v}_{d}(\Vec{x}) = [ 1\text{ } x_1\text{ } x_2\text{ } x_{1}x_{2}\text{ } x_{1}^2 \text{ }x_{2}^2 ] $.

\subsection{Christoffel Functions}
\label{Christoffel Functions}
Christoffel functions are a class of functions associated with a finite measure and a parameter degree $d \in \mathbb{N}$. They have a strong connection to approximation theory and in this section we briefly summarize some results by \cite{lasserre2019empirical}. 
For a finite measure $\mu$ on $\mathbb{R}^{n}$ and an integer degree $d$, the \emph{Christoffel function} $\Lambda_{\mu, d}( \Vec{x}):\mathbb{R}^n \mapsto \mathbb{R}$ is defined in terms of the \emph{moment matrix} of the measure $\mu$: 
$$
\mathbf{M}_d=\int_{\mathbb{R}^{n}} \Vec{v}_{d}(\Vec{x}) \Vec{v}_{d}(\Vec{x})^{\top} d \mu(\Vec{x}).$$

The moment matrix is semi-definite positive for all $d \in \mathbb{N}$. We furthermore assume that the matrix is positive definite, which ensures the invertibility of $M_d$.\footnote{In fact, the moment matrix of any finite measure $\mu$ is definite positive unless the support of $\mu$ is contained in the zeros of a polynomial; for a closer look at the moment matrix, we refer the reader to \cite{lasserre2019empirical}.} With the help of the moment matrix, the Christoffel function is defined as : 
\begin{equation}\label{eq:Christoffel}
\Lambda_{\mu, d}(\Vec{x}) = \Bigl(\mathbf{v}_{d}(\mathbf{\Vec{x}})^{T} \mathbf{M}_{d}^{-1} \mathbf{v}_{d}(\mathbf{\Vec{x}})\Bigr)^{-1}.
\end{equation}
The following alternative formulation of the Christoffel function can be useful when the moment matrix is large. It can be computed by solving a convex quadratic programming problem, which can be done efficiently using numerical techniques, even for high degrees $d$: 
$$ 
\Lambda_{\mu, d}( \Vec{x}) = \inf _{P \in \mathbb{R}[\Vec{X}]_d^n}\left\{\int_{\mathbb{R}^{n}} P(\mathbf{z})^{2} d \mu(\mathbf{z}), \quad \text {s.t.} \quad P(\mathbf{x})=1\right\}
$$

In a data-driven setting, the exact measure $\mu$ is unknown. One way to obtain information about $\mu$ is by sampling a set of points independently drawn from its distribution. For every $N \in \mathbb{N}$, when disposing of $N$ i.i.d samples $\{ \Vec{x}^{1}, \ldots, \Vec{x}^{N} \}$ from $\mu$, we approximate $\mu$ with the \emph{empirical measure} %$\hat{\mu}$ %defined as follows :
\[  \hat{\mu} = \tfrac{1}{N}{\sum}_{i=1}^N \delta_{\Vec{x}^i},\]
where $\delta_{\Vec{x}}$ is the Dirac measure.
The moment matrix $\widehat{\mathbf{M}}_d$ associated with the empirical measure $\hat{\mu}$ is 
\begin{equation}\label{eq:moment matrix}
\widehat{\mathbf{M}}_d = \tfrac{1}{N} {\sum}_{i=1}^{N}\mathbf{v}_{d}(\mathbf{x^{i}}) \mathbf{v}_{d}(\mathbf{x^{i}})^{T} 
\end{equation}
Therefore, the empirical measure $\hat{\mu} $ defines an empirical Christoffel function. Since we are only interested in superlevel sets of the Christoffel function, we can forego the inversion and instead work with sublevel sets of what we call the empirical \emph{Christoffel polynomial}: \begin{equation}\label{eq:empiricalChristoffel}
 \Lambda^{-1}_{\hat{\mu}, d}(\Vec{x}) = \mathbf{v}_{d}(\Vec{x})^{T} \widehat{\mathbf{M}_{d}}^{-1} \mathbf{v}_{d}(\Vec{x})
\end{equation}
Note that the moment matrix $\widehat{\mathbf{M}}_d$ is almost surely invertible if the number of samples $N \geq s(d)$.
The Christoffel polynomial is a sum-of-squares polynomial of degree $2d$. Consequently, it is nonnegative, and if $N > s(d)$, the empirical Christoffel polynomial is strictly positive.
Note that, for increasing sample size $N$, the empirical Christoffel function converges uniformly to the Christoffel function of the exact measure.

\subsection{Set Approximation with Christoffel Functions}
\label{ssec:Set Approximation with Christoffel Functions}
\cite{lasserre2019empirical} proposed various thresholding schemes for approximating the support of a probability measure using the Christoffel function or, more precisely, its empirical counterpart. This idea was applied by \cite{devonport} to approximate the reachable set $\set{S}$ with the superlevel sets of the Christoffel function. In this section, we will briefly summarize the approach. 

Let $\mu$ be the probability measure of the reachable set $\set{S}$. For a given degree $d \in \mathbb{N}$, the reachable set can be approximated with the sublevel set 
\begin{equation}\label{eq:cpls}
\hat{\set{S}} = \{ \Vec{x} \in \mathbb{R}^{n} \mid \Lambda^{-1}_{\mu, d}(\Vec{x}) \leq \alpha \}
\end{equation}
for some $\alpha \in \mathbb{R}$. However, since the exactly reachable set $\set{S}$ is unknown, $\mu$ is unknown. Instead, the Christoffel function $\Lambda_{\mu, d}$ is approximated by an empirical Christoffel function using i.i.d generated samples $\Vec{x^{i}}$ from $\set{S}$. We can obtain a conservative threshold $\alpha$ such that $\Vec{x^{i}} \subseteq \hat{\set{S}}$ by letting
\begin{equation}\label{eq:maxalpha}
\alpha = \max_i \Lambda^{-1}_{\hat{\mu}, d}(\Vec{x}^{i}).
\end{equation}Using methods from statistical learning theory, \cite{devonport} proposed the following PAC guarantees:
\begin{conjecture}[Thm.\;1 in \cite{devonport}]\label{thm:devonport_PAC}
Given a training set of i.i.d samples $\set{D} = \{ \Vec{x}^1, \ldots, \Vec{x}^N\}$ from $\set{S}$, let 
\begin{equation}\label{eq:maxalphaset}
\hat{\set{S}} = \{ \Vec{x} \in \mathbb{R}^{n} \mid \Lambda^{-1}_{\hat{\mu}, d}(\Vec{x}) \leq \max_i \Lambda^{-1}_{\hat{\mu}, d}(\Vec{x}^{i}) \}.
\end{equation}If $N \geq \frac{5}{\epsilon}\left(\log \frac{4}{\delta}+\binom{n+2d}{n} \log \frac{40}{\epsilon}\right)$, then $\mathbb{P}\Bigl( \mu \bigl( \hat{\set{S}} \bigr) \geq 1-\epsilon \Bigr) \geq 1-\delta.$
\end{conjecture}
%\vspace{-1cm}
In other words, if $N,\delta,\epsilon$ satisfy the condition in Conjecture\;\ref{thm:devonport_PAC}, then with probability bigger than $1-\delta$ we are sure that $\hat{\set{S}}$ contains more than $1-\epsilon$ of the mass of $\set{S}$. 
However, we believe this result neglects the dependencies between the empirical Christoffel polynomial and the points used to construct the threshold $\alpha$. As will be discussed in more detail in \sectionref{Reach Set Approximation with Conformal Prediction}, different samples should be used for  constructing the empirical Christoffel polynomial and for constructing the threshold $\alpha$ to ensure independence. 
%The reason behind this will be explained in detail in \sectionref{Reach Set Approximation with Conformal Prediction}, where we will also use conformal prediction theory to provide an algorithm that requires a much smaller number of samples for the same type of guarantees. Also, this number of samples will depend only on the approximation uncertainty $\epsilon$ and the confidence $\delta$ and not on the degree $d$ and the input dimension $n$.

We informally note convergence results by \cite{lasserre2019empirical}, which hold for uniform probability measures (and some generalizations):
\begin{itemize}
\itemsep=0ex
    \item As $d\rightarrow \infty$ and with an appropriately chosen threshold, the sublevel set of the (non-empirical) Christoffel polynomial converges to the support of the measure, i.e., to the exact reach set in these sense of a Hausdorff distance.
    \item For fixed $d$ and $n\rightarrow \infty$, the empirical Christoffel function converges uniformly to the Christoffel function.
    \item For fixed $d$ and $n\rightarrow \infty$, the border of the empirical Christoffel polynomial converges to the border of the Christoffel polynomial in these sense of a Hausdorff distance. 
\end{itemize}
In consequence, we can informally expect that for a large enough degree $d$ and large enough sample size $N$, the sublevel sets of the Christoffel polynomial are close enough to the reachable set.

We will use the following running example throughout the paper to illustrate the different concepts. 
\begin{example}[Four squares]\label{ex:foursquares}
Let the transition function $f : \mathbb{R}^2 \rightarrow \mathbb{R}^2$ be $$f(x,y) = (1+sign(x)\cdot x^2,1+sign(y)\cdot y^2)$$
and let the initial set be $\set{I} = [-1,1]^2$. The reachable set consists of four squares, i.e., $$\set{S}= [-3,-1]^2 \cup [-3,-1]\times [1,3] \cup [1,3]\times [-3,-1] \cup [1,3]^2.$$
Figure\;\ref{fig:figure_1} shows the reach set approximation given by \eqref{eq:maxalphaset}, for a sample of size $N=10\,000$ and different degrees $d$.
The caption includes the corresponding uncertainty bound $\varepsilon$ for confidence $1-\delta = 0.99$ obtained by Conjecture\;\ref{thm:devonport_PAC}.

We observe that, as intended by construction, all samples are included in $\hat S$. For increasing degrees, $\hat S$ becomes more precise. However, the uncertainty in the covered probability mass $\epsilon$, increases substantially. Indeed, the bound $\epsilon$ seems rather conservative since, in all instances, $\hat S$ covers nearly $100\%$ of $S$.
\end{example}
 \begin{figure}[t!bp]
\floatconts
  {fig:figure_1}
  {\caption{Reach set approximation $\hat S$ for Example \ref{ex:foursquares}, using the sublevel set of the empirical Christoffel polynomial in \eqref{eq:maxalphaset} (purple outline) on a sample of size $N=10000$ (black dots), for different degrees $d$ and corresponding uncertainty bound $\varepsilon$, according to Conjecture\;\ref{thm:devonport_PAC}.}}
  {%
    \subfigure[$d=3$, $\varepsilon=0.085$]{\label{fig:image-a-pac}%
      \includegraphics[width=0.245\linewidth]{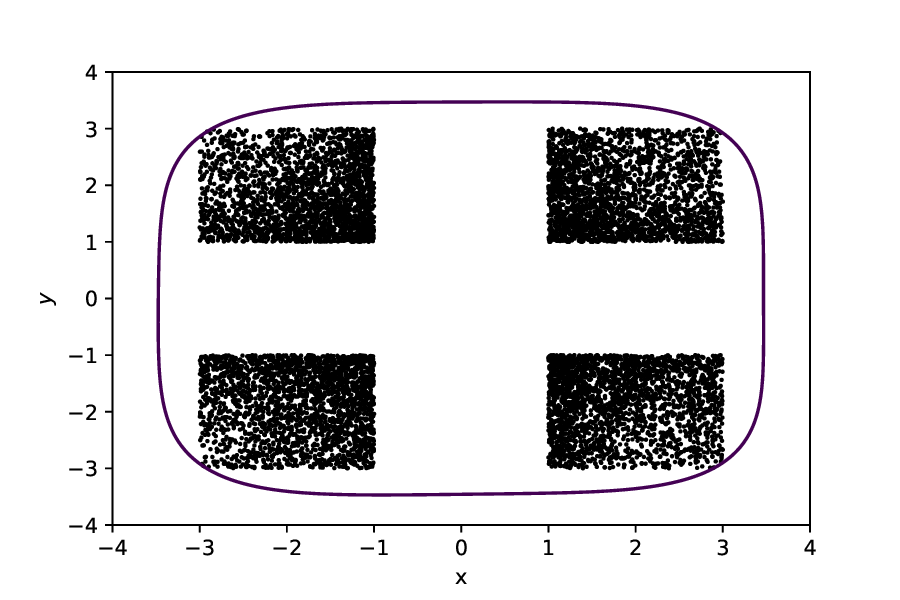}}%
    %\qquad
    \subfigure[$d=6$, $\varepsilon=0.23$]{\label{fig:image-b-pac}%
      \includegraphics[width=0.245\linewidth]{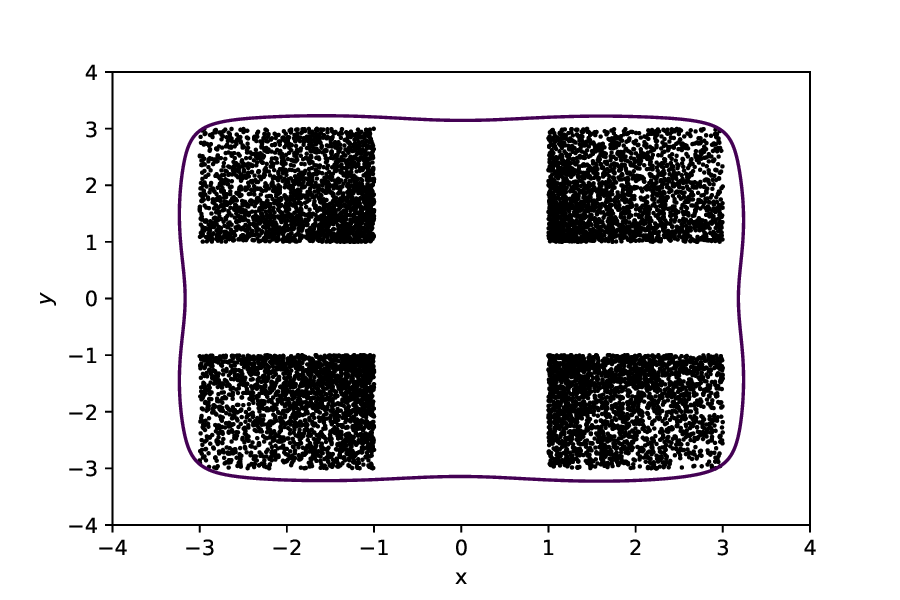}}    %\qquad
    \subfigure[$d=10$, $\varepsilon=0.51$]{\label{fig:image-c-pac}%
      \includegraphics[width=0.245\linewidth]{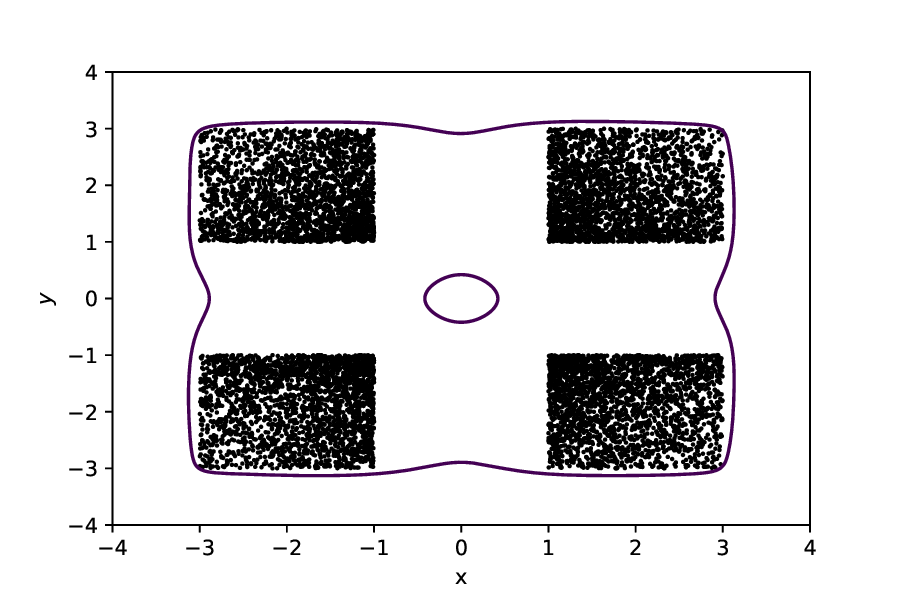}}   % \qquad
    \subfigure[$d=15$, $\varepsilon=0.9$]{\label{fig:image-d-pac}%
      \includegraphics[width=0.245\linewidth]{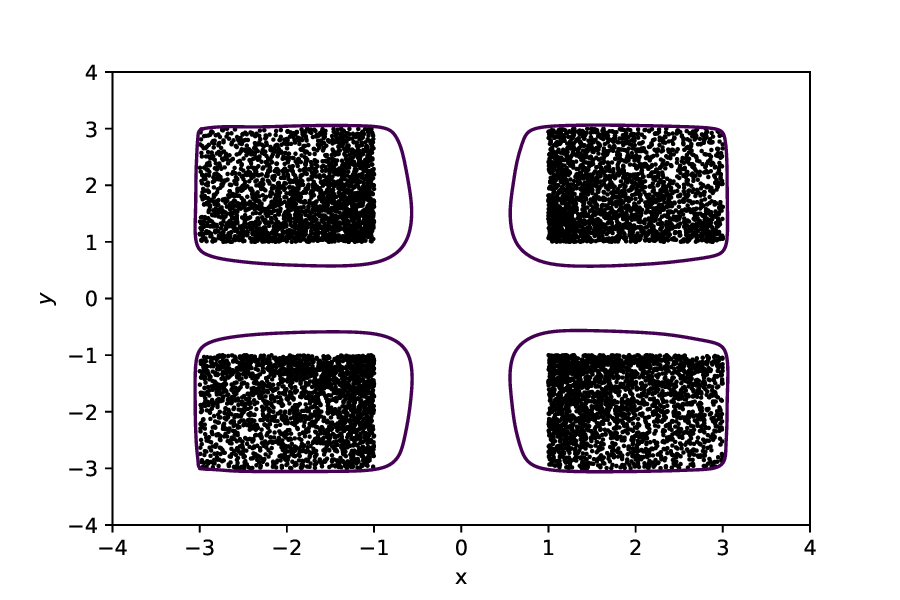}}
  }
\end{figure}
%
%
%\todo[inline]{We must replace "approximation error" with uncertainty, since it's not the same thing!!! The uncertainty can be large even though the actual error isn't, line in Figure 1(d)!}
%

%
%\todo[inline]{Include convergence with $d\rightarrow \infty$}
%
\section{Reach Set Approximation with Conformal Prediction}\label{Reach Set Approximation with Conformal Prediction}
Following the reasoning of \sectionref{ssec:Reach-Set-Approximation}, we can expect a sublevel of the Christoffel polynomial to converge to the support of the distribution. Intuitively, the Christoffel polynomial takes high values where the density is low and low values where the density is high, which makes it a good candidate for a nonconformity function.

In this section, we briefly recall relevant results from conformal prediction and instantiate them to the special case of estimating the support of distribution, which in our setting is equivalent to approximating the reach set $\set{S}$.
Let $r: \mathbb{R}^n \rightarrow \mathbb{R}$ be a non-conformity function. %Consider $N$ i.i.d random variables $\Vec{x}^{(i)} \in \mathbb{R}^n$ with $i \in \{1,...,N\} $ sampled from some probability distribution $\mu$,we note $\set{D} = \{ \Vec{x}^{(i)}$ for $i \in \{1,...,N\}$. For $\Vec{x} \in \mathbb{R}^n$, 
Given a sample $\set{D} = \bigl\{ \Vec{x}^{1},\ldots, \Vec{x}^{N}\bigr\}$,
the $p$-value at $\Vec{x}$ is % defined as :
\[  p_{value}(\Vec{x})=\tfrac{1}{N}\Bigl|\bigl\{\ i  \bigm| r(\Vec{x}^i) \geq r(\Vec{x})\;\bigr\}\Bigr| \]
For $i \in \{0,..., N \}$, the conformal region is defined as
%follows :
\[ 
C_{\set{D}}^{\frac{i}{N}}=\Bigl\{ \Vec{x} \in \mathbb{R}^n \Bigm| p_{value}(\Vec{x}) \geq \tfrac{i}{N}\Bigr\}
  \]
According to conformal prediction theory, see \cite{vovk,Anastasios}, a new i.i.d sample $\Vec{x}^{N+1}$ satisfies 
\begin{equation}\label{eq:Conformal prediction}
\mathbb{P}\left(\Vec{x}^{N+1} \in C_{\set{D}}^{\frac{i}{N}}\right) \geq 1-\tfrac{i+1}{N+1}.
\end{equation}
Note that in \eqref{eq:Conformal prediction}, the set $\set{D}$ is also subject to randomness. In other words, \eqref{eq:Conformal prediction} stands on average only if the set $\set{D}$ is re-sampled for each $\Vec{x}^{N+1}$. However, in reachability analysis and data-driven applications more generally, we may be restricted to a single, fixed data set $\set{D}$.  
Therefore, we need to take into account the probability on the left hand side of \eqref{eq:Conformal prediction}, conditioned on the sample $\set{D}$.

\subsection{Statistical Guarantees}
\label{ssec:Statistical Guarantees}
In this section, we ensure statistical independence between the nonconformity function $r$ and the set $\set{D}$
by splitting it into a training set $\set{D}_\mathrm{train}$ and a calibration set $\set{D}_\mathrm{cal}$. %
The use of distinct sets of samples from the same measurement (i.e., a training set and a calibration set) is essential to ensure the independence of the samples used for computing the p-values and conformal regions from the nonconformity function, which is computed based on the training set, see  \cite{Anastasios} and \cite{bates}. This is a special case of conformal prediction called split conformal prediction or inductive conformal prediction. The computational advantage of this method lies in its requirement to fit the model only once. However, this comes at the cost of statistical efficiency as the method necessitates the division of the data into separate, and therefore smaller, training and calibration data sets. A smaller calibration set increases the coverage error, while a smaller training set reduces the tightness of the approximation. An alternative to this trade-off will be examined in section \ref{ssec:Avoiding the Calibration Set}.
%, following an approach by \cite{bates}. 
Here, we use the training set for computing the empirical Christoffel polynomial, while the calibration set is used to compute the conformal region. This will lead to bounds on the conditional probability
$$ \mathbb{P}\Bigl(\Vec{x}^{N+1} \in C_{\set{D}}^{\frac{i}{N}}\ \Bigm| \set{D}_\mathrm{cal} \Bigr).$$ 
Note that the theorems in this section %, unless specified otherwise, 
apply to any choice of nonconformity function.
The following theorem %from \cite{bates} 
provides PAC guarantees for conformal regions that are defined with suitably chosen probability thresholds $b_1,\ldots,b_n$.
We will afterwards propose values for $b_1,\ldots,b_n$ that correspond to the special case of set approximation.
 \begin{theorem}[Thm.\;4 from \cite{bates}]\label{thm:calibration}
Consider $N$ uniform random samples $U_1, \ldots, U_N \stackrel{\text { i.i.d. }}{\sim} \operatorname{Unif}\bigl([0,1]\bigr)$, with order statistics $U_{(1)} \leq U_{(2)} \leq \ldots \leq U_{(N)}$, and fix any $\delta \in(0,1)$. Suppose $0 \leq b_1 \leq b_2 \leq \ldots \leq b_N \leq 1$ are reals such that 
$$\mathbb{P}\left[U_{(1)} \leq b_1, \ldots, U_{(N)} \leq b_n\right] \geq 1-\delta. $$
Let also $b_0=0$. %Let $\mu$ be a measure and $\set{D} = \{ \Vec{x}^{(i)}$ for $i \in \{1,...,N\}\} $ are N i.i.d samples from $\mu$ 
Then for any i.i.d vector $\Vec{x}$ sampled from $\mu$:
\begin{equation}\label{eq:bates}
    \mathbb{P}\left[\mathbb{P}\Bigl( \Vec{x} \in C_{\set{D}_\mathrm{cal}}^{\frac{i}{N}} \Bigm| \set{D}_\mathrm{cal}\Bigr) \geq 1 - b_i \right] \geq 1-\delta
\end{equation}
 \end{theorem}
 %
%\todo[inline]{Does $\mu$ need to be continuous in Thm 2???}
%
We propose an analogous, theorem to bound the conditional probability from above.
\begin{theorem}\label{thm:Tebjou}
%Let $U_1, \ldots, U_N \stackrel{\text { i.i.d. }}{\sim} \operatorname{Unif}([0,1])$, with order statistics $U_{(1)} \leq U_{(2)} \leq \ldots \leq U_{(N)}$, and fix any $\delta \in(0,1)$.
Under the assumptions of Thm.\;\ref{thm:calibration}, 
suppose further that the nonconformity function $r(\Vec{x})$ is continuous, the measure $\mu$ is continuous, and that 
$\alpha$ is a real such that 
$$\mathbb{P}\bigl(U_{(N)} \leq \alpha \bigr) \geq 1-\delta. $$ 
%Let $\mu$ be a continuous measure and $\set{D} = \{ \Vec{x}^{(i)}$ for $i \in \{1,...,N\}\} $ are N i.i.d samples from $\mu$ 
Then for any i.i.d vector $\Vec{x}$ sampled from $\mu$:
\begin{equation}\label{eq:tebjou}
    \mathbb{P}\left[\mathbb{P}\Bigl( \Vec{x} \in C_{\set{D}_\mathrm{cal}}^{\frac{1}{N}} \Bigm| \set{D}_\mathrm{cal}\Bigr)  \leq \alpha \right] \geq 1-\delta
\end{equation}
%With $C_{\set{D}}^{\frac{1}{N}} $ the conformal region computed using a Christoffel polynomial as the non-conformity function $r$.
\end{theorem}
\begin{proof}
%
%    Let's consider the set $\set{D} = \{ \Vec{x}^{(i)}$ for $i \in \{1,...,N\}\} $. Assume that $\mu$ has a continuous distribution. Let X be a random variable following the measure $\mu$.
Under the assumptions, $r(\Vec{x})$ has a continuous distribution.
%because we chose the nonconformity function $r$ to be the Christoffel polynomial. 
Let $F_{\mu}$ be the cumulative distribution function of $r(\Vec{x})$. Since $r(\Vec{x})$ has a continuous distribution, $F_{\mu}(r(\Vec{x}))$ follows $\operatorname{Unif}([0,1])$ and $F_{\mu}(r(\Vec{x}^1)),F_{\mu}(r(\Vec{x}^2)),...,F_{\mu}(r(\Vec{x}^N))$ all follow $\operatorname{Unif}([0,1])$.
Without loss of generality, we assume $r(\Vec{x}^1) \leq r(\Vec{x}^2) \leq ... \leq r(\Vec{x}^N)$. Letting $U_N = F_{\mu}(r(\Vec{x}^N))$, we obtain 
   $$\mathbb{P}\Bigl[ F_{\mu}(r(\Vec{x}^N)) \leq \alpha \Bigr] \geq 1-\delta.$$ 
    Considering $\Vec{x}$ sampled i.i.d from $\mu$, we get
    \[ \mathbb{P}\Bigl( \Vec{x} \in C_{\set{D}_\mathrm{cal}}^{\frac{1}{N}} \Bigm| \set{D}_\mathrm{cal}\Bigr) = \mathbb{P}\Bigl(r(\Vec{x}) \leq r(\Vec{x}^N) \Bigm| \set{D}_\mathrm{cal}\Bigr) = F_{\mu}\Bigl(r(\Vec{x}^N)\Bigr) \] 
    Combining the latter two results, we obtain \eqref{eq:tebjou}.
\end{proof}
We now use the results of Thm.\;\ref{thm:calibration} and Thm.\;\ref{thm:Tebjou}
%\eqref{eq:bates},\eqref{eq:tebjou} 
to provide a guarantee on the accuracy of approximated reachable set $\hat{\set{S}}$ in Algorithm 1. Note that  Thm.\;\ref{thm:Tebjou} requires the nonconformity function to be continuous, which is the case for the empirical Christoffel polynomial.
\begin{theorem}\label{thm:calibrated_sublevels}
%Let $\mu$ be a measure, let $\set{D} = \{ \Vec{x}^1, \ldots, \Vec{x}^N\} $ be $N$ i.i.d samples from $\mu$, and let $C_{\set{D}}^{\frac{1}{N}}$ be the conformal region computed using a Christoffel polynomial as the non-conformity function, then $\forall \delta \in (0,1)$  :
Suppose that the nonconformity function $r(\Vec{x})$ is continuous. $\forall \delta \in (0,1),$
\begin{equation}\label{eq:final1}
\mathbb{P}\left[  \mu \left(   C_{\set{D}_\mathrm{cal}}^{\frac{1}{N}} \right) \geq  \exp \left(\frac{ \log (\delta)}{N}\right) \right] \geq 1-\delta,
\end{equation}
If the measure $\mu$ is continuous, then
\begin{equation}\label{eq:final2}
\mathbb{P}\left[  \exp \left(\frac{ \log ( 1 - \delta)}{N}\right)  \geq \mu \left(   C_{\set{D}_\mathrm{cal}}^{\frac{1}{N}}   \right)  \right] \geq 1-\delta  ,
\end{equation}
Combining these results, we obtain $\forall \delta \in (0,\sfrac{1}{2})$:
\begin{equation}\label{eq:final3}
    \mathbb{P}\left[  \exp \left(\frac{ \log ( 1 - \delta)}{N}\right)  \geq \mu \left(   C_{\set{D}_\mathrm{cal}}^{\frac{1}{N}}   \right) \geq  \exp \left(\frac{ \log (\delta)}{N}\right) \right] \geq 1-2\delta. 
\end{equation}

\end{theorem}
\begin{proof}
We instantiate \theoremref{thm:calibration} for a particular choice of $b_1 \ldots, b_N$. Since we are interested in the support of the measure, we take $b_1$ as the smallest possible value and set the other values $b_2 \ldots, b_N = 1$.
To satisfy the conditions of \theoremref{thm:calibration}, we first show the following intermediate result:
Let $U_1, \ldots, U_N \stackrel{\text { i.i.d. }}{\sim} \operatorname{Unif}([0,1])$, with order statistics $U_{(1)} \leq U_{(2)} \leq \ldots \leq U_{(N)}$.
Fixing $b_1 = 1- \delta^{\frac{ 1}{N}} $ and $b_2 = ....= b_N = 1$, it is straightforward that
$$ \mathbb{P}\left(U_{(1)} \leq b_1, \ldots, U_{N} \leq b_N\right) = \mathbb{P}\left(U_{(1)} \leq b_1 \right) = 1-\mathbb{P}\left(U_{(1)} \geq b_1 \right).$$
Since $U_{(1)}$ is the smallest of the random variables, $U_1, \ldots, U_N$, $\mathbb{P}\left(U_{(1)} \geq b_1 \right)$ is equivalent to all of the $U_i$ being greater or equal to $b_1$: 
%$$\mathbb{P}\left(U_{1} \leq b_1 \right) = 1- \mathbb{P}\left(U_{1} \geq b_1 \right) = 1- (1- b_1)^N \geq 1-\delta $$
%
%
$$1-\mathbb{P}\left(U_{(1)} \geq b_1 \right) = 1 -  \Pi_{i=1}^N \mathbb{P}\left(U_{i} \geq b_1 \right) = 1- (1- b_1)^N = 1-\delta.$$
Applying the above in \theoremref{thm:calibration}, we obtain 
%for any  $\Vec{x}$ sampled i.i.d. from $\mu$ that\[ \mathbb{P}\left[\mathbb{P}\Bigl( \Vec{x} \in C_{\set{D}}^{\frac{1}{N}} \Bigm| \set{D}\Bigr) \geq 1 - b_1 \right] \geq 1-\delta\]
\[ \mathbb{P}\left[\mathbb{P}\Bigl( \Vec{x} \in C_{\set{D}_\mathrm{cal}}^{\frac{1}{N}} \Bigm| \set{D}_\mathrm{cal}\Bigr) \geq \exp \left(\frac{ \log (\delta)}{N}\right) \right] \geq 1-\delta.\]
As $ \mu \left(   C_{\set{D}_\mathrm{cal}}^{\frac{1}{N}}   \right) = \mathbb{P}\left[ \Vec{x} \in C_{\set{D}_\mathrm{cal}}^{\frac{1}{N}} \mid \set{D}_\mathrm{cal}\right]$ we obtain the result in \eqref{eq:final1}.
%
%
%Therefore we have :
%$$\mathbb{P}\left[U_{(1)} \leq b_1, \ldots, U_{(N)} \leq b_n\right] \geq 1-\delta $$

Fixing $\alpha = \exp \left(\frac{ \log ( 1 - \delta)}{N}\right)$, we have $ \mathbb{P}\left[U_{N} \leq \alpha \right] = \alpha^N =1-\delta$, since $U_{(N)} \leq \alpha$ means all $U_{i}$ have to be lower than $\alpha$.
 %, we get : 
%\[ \mathbb{P}\left[U_{(N)} \leq \alpha \right] = \alpha^N = \left(\exp \left(\frac{ \log ( 1 - \delta)}{N}\right)\right)^N \geq 1-\delta \]
%Therefore we have :
%$$\mathbb{P}\left[U_{(N)} \leq \alpha \right] \geq 1-\delta $$ 
Substituting the above value of $\alpha$ in \theoremref{thm:Tebjou}, we obtain the result in \eqref{eq:final2}.
%
%\[ \mathbb{P}\left[  \exp \left(\frac{ \log ( 1 - \delta)}{N}\right)  \geq \mu \left(   C_{\set{D}}^{\frac{1}{N}}   \right)  \right] \geq 1-\delta  \]
%
Combing \eqref{eq:final1} and \eqref{eq:final2}, we obtain the result in \eqref{eq:final3}.
%
%\todo[inline]{GF: clarify instead of repeating (11) and (12)}
%\todo[inline]{GF: it's not clear where the term $\exp \left(\frac{ \log ( 1 - \delta)}{N}\right)$ comes from}
%\todo[inline]{GF: does it hold only for the Christoffel polynomial??????}

\end{proof}

\begin{algorithm2e}
\caption{Reach set approximation (without outliers)}
\label{alg:Reachability analysis without outliers}

\KwIn{
%Transition function $f$; initial set $\set{I} \subset \mathbb{R}^n$; 
An i.i.d data sample 
$\set{D} = \{\Vec{x}^{1},\ldots, \Vec{x}^{M}\}$, drawn from the reach set $\set{S}=f(\set{I})$, %$\Vec{x}^{(i)} \in f(\set{I})$;
the degree $d$, the size $N$ of the calibration set with $N<M$}
\KwOut{$\epsilon$-accurate approximation $\hat{\set{S}}$ of $\set{S}$ with confidence $1-\delta$ and coverage error $\epsilon = 1- \delta^{\sfrac{1}{N}}$
%\exp \left(-\frac{ \log (1 / \delta)}{N}\right)$
}

%\For{$i \in \{1,...,M \}  $}{
   %\sout{Observe i.i.d. samples $\Vec{y^{(i)}} \in \mathbb{R}^n$. \;} \texttt{\#In/Out.lier detection:}
   
    %The sample $\Vec{y^{(i)}}$ is an \emph{inlier} if $\Vec{y^{(i)}} \in f(\set{I})$ and an \emph{outlier} otherwise. 
%  Generate i.i.d samples $\Vec{x^{(i)}} \in \set{I}$; %and $u^i \in \set{U}$ \;
% Compute $f(\Vec{x^{(i)}})$ \;
%}
\begin{enumerate}
  \itemsep=0ex
  \item[\texttt{\#}] Construct the training set of $M-N$ samples and the calibration set of $N$ samples:  
  $\set{D}_\mathrm{train} = \{\Vec{x}^{N+1},\ldots, \Vec{x}^{M}\}$ 
      %\item[\texttt{//}] Construct the calibration set of $N$ samples: 
      and $\set{D}_\mathrm{cal} = \{ \Vec{x}^{1},\ldots, \Vec{x}^{N}\}$ 
  \item Compute the empirical moment matrix $\widehat{\mathbf{M}}_d$ and its inverse %the associated Christoffel polynomial :
  %\item We use $M-N$ samples from $\set{D}_\mathrm{train} = \{\Vec{x}^{N+1},\ldots, \Vec{x}^{M}\}$ to compute the empirical moment matrix and invert it %the associated Christoffel polynomial :
    \begin{enumerate}
      \itemsep=0ex
  \item $ \widehat{\mathbf{M}}_d =\frac{1}{M-N} \sum_{i=N+1}^{M} \mathbf{v}_{d}\left(\Vec{x}^{i}\right) \mathbf{v}_{d}\left(\Vec{x}^{i}\right)^{\top}$, with $\Vec{x}^{i}\in \set{D}_\mathrm{train}$
      \item Compute $\widehat{\mathbf{M}}^{-1}_d$.
       \end{enumerate}

     \item Calculate the threshold $\alpha$: $ \alpha =\max _{i=1, \ldots, N} \mathbf{v}_{d}\left(\Vec{x}^{i}\right)^{\top} {\widehat{\mathbf{M}}}^{-1}_d  \mathbf{v}_{d}\left(\Vec{x}^{i}\right)$, with $\Vec{x}^{i}\in \set{D}_\mathrm{cal}$
     %Return $\hat{\set{S}}$ as the conformal region $C_{\set{D}_\mathrm{train}}^{\frac{1}{N}}$ :
%     \item Using $N$ samples $\set{D}_\mathrm{cal} = \{ \Vec{x}^{1},\ldots, \Vec{x}^{N}\}$ as a calibration set, return $\hat{\set{S}}$ as the conformal region $C_{\set{D}_\mathrm{train}}^{\frac{1}{N}}$ :
%      \begin{enumerate}
%          \item $ \alpha =\max _{i=1, \ldots, N} v_{d}\left(\Vec{x}^{i}\right)^{\top} {\widehat{\mathbf{M}}}^{-1}_d  v_{d}\left(\Vec{x}^{i}\right)$, with $\Vec{x}^{i}\in \set{D}_\mathrm{cal}$
%          \item $ \hat{\set{S}}  = \Bigl\{\Vec{x} \in \mathbb{R}^{n} \Bigm| v_{d}(\Vec{x})^{\top} {\widehat{\mathbf{M}}}^{-1}_d  v_{d}(\Vec{x}) \leq \alpha  \Bigr\}   $
%      \end{enumerate}
 \item Given the returned  
 ${\widehat{\mathbf{M}}}^{-1}_d$ and 
 $\alpha$, record 
the conformal region: %$C_{\set{D}_\mathrm{cal}}^{\frac{1}{N}}$ 
 \[C_{\set{D}_\mathrm{cal}}^{\frac{1}{N}}= \hat{\set{S}}  = \Bigl\{\Vec{x} \in \mathbb{R}^{n} \Bigm| \mathbf{v}_{d}(\Vec{x})^{\top} {\widehat{\mathbf{M}}}^{-1}_d  \mathbf{v}_{d}(\Vec{x}) \leq \alpha  \Bigr\}\]
\end{enumerate}
\end{algorithm2e}

\begin{example}\label{ex:example2}
We illustrate \algorithmref{alg:Reachability analysis without outliers} on the running Example \ref{ex:foursquares}.
%
%We conduct the following experiments to evaluate and visualize the performance . 
We take $M=10000$ i.i.d samples from the reachable set $\set{S}$ by sampling uniformly $M$ i.i.d samples in $\set{I}$, which we then split into a calibration set of size $N=2000$ and a training set of size $M-N$. 
%If we set the initial set $\set{I} = [-1,1]^2$ the reachable set will be $\set{S}= [-3,-1]^2 \cup [-3,-1]\times [1,3] \cup [1,3]\times [-3,-1] \cup [1,3]^2$. 
%We compute $M=10000$ samples in $\set{S}$ by sampling uniformly $M$ i.i.d samples in $\set{I}$. We take $N=2000$ samples as a calibration set and the remaining $M-N= 8000$ samples as a training set used to compute the empirical Christoffel polynomial. 
\figureref{fig:Calibrated_sublevels_M=10000} shows the approximated reachable set produced by \algorithmref{alg:Reachability analysis without outliers} for various degrees $d$. \theoremref{thm:calibrated_sublevels} guarantees that with confidence $1-\delta = 99\%$, the coverage error  $\epsilon$ is lower than $ \epsilon \leq 0.002 $. Notably, in contrast to the algorithm presented in \cite{devonport}, this guarantee is independent of the dimension of the samples $n$ and the degree of the empirical Christoffel polynomial $d$. It only depends on the confidence parameter $\delta$ and the size $N$ of the calibration set.
\begin{figure}[tbp]
\floatconts
{fig:Calibrated_sublevels_M=10000}
  {\caption{Reach set approximations (outlined in purple) from Example \ref{ex:example2},
  %of the reach set %in of example \ref{ex:foursquares}, 
  obtained with \algorithmref{alg:Reachability analysis without outliers}, which uses the Christoffel polynomial as a nonconformity function, for $M=10000$ samples, of which $N=2000$ are the calibration set (red dots) and the remainder the training set (black dots). 
  Higher degrees $d$ lead to tighter approximation.
  }}
  {\centering%
    \subfigure[$d=6$, $\varepsilon=0.002$]{\label{fig:image-a-cs}   \includegraphics[width=0.32\linewidth]{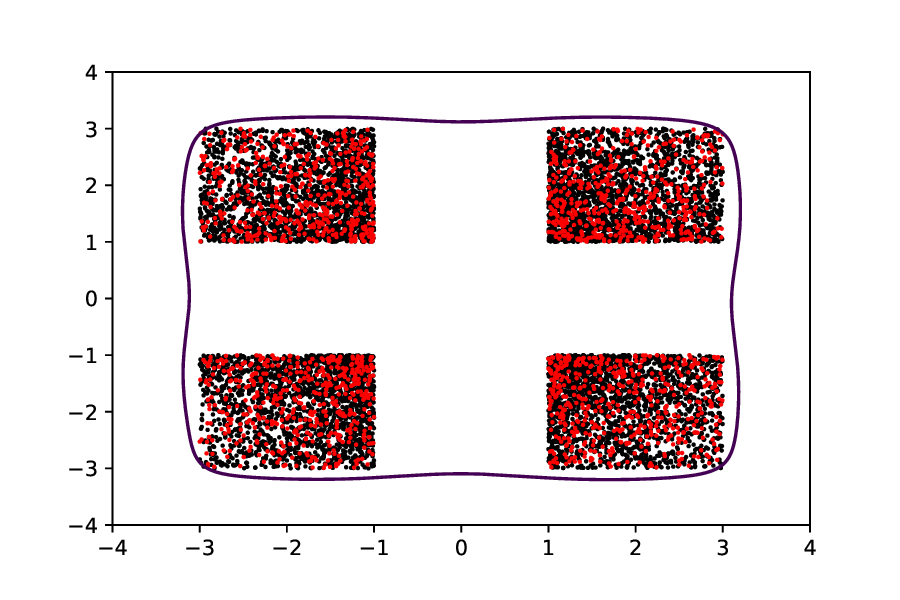}}%
    %\qquad
    \subfigure[$d=10$, $\varepsilon=0.002$]{\label{fig:image-b-cs}%
\includegraphics[width=0.32\linewidth]{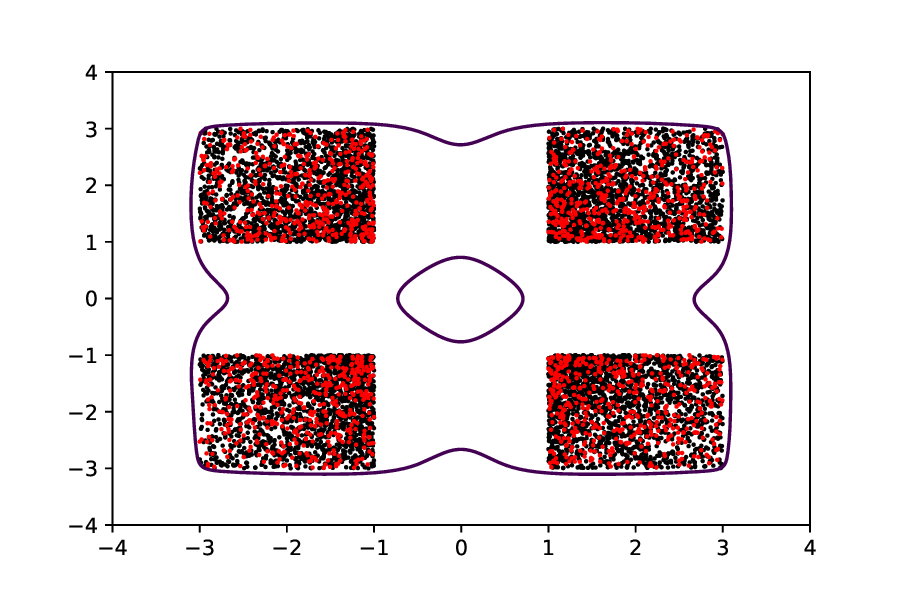}}  %\qquad
    \subfigure[$d=15$, $\varepsilon=0.002$]{\label{fig:image-c-cs}%
\includegraphics[width=0.32\linewidth]{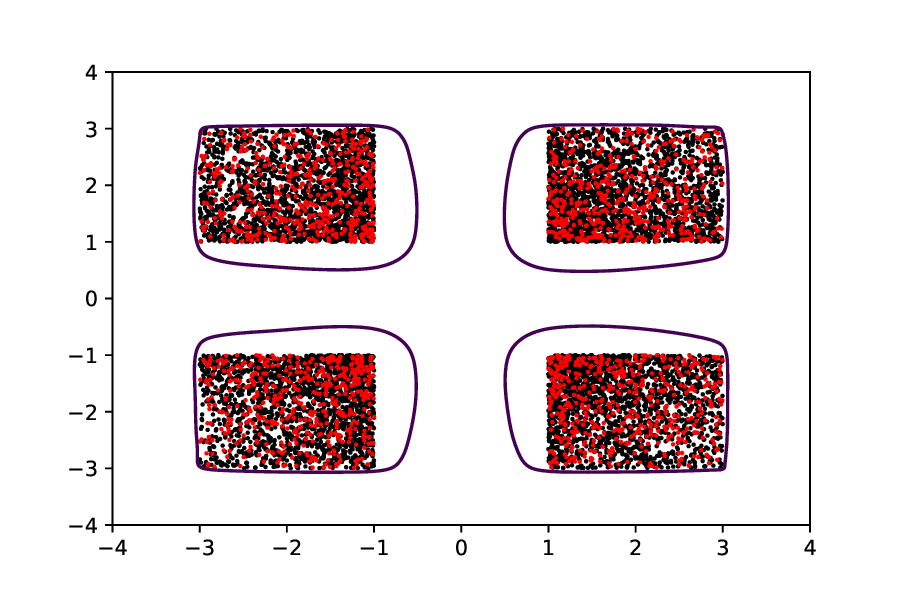}}   % \qquad
  }
\end{figure}
\figureref{fig:Calibrated_sublevels_M=1000} shows the same result for  $M=1000$ samples, of which $N=200$ samples were utilized as a calibration set. % and the remaining $M-N$ samples were used as a training set. \theoremref{thm:calibrated_sublevels} guarantees that with confidence $1-\delta = 99\%$,
Here, the coverage error $\epsilon$ will be lower than $ \epsilon \leq 0.02 $.
To empirically verify the theoretical guarantees obtained in \theoremref{thm:calibrated_sublevels}, we repeated this experiment 1000 times. The empirical error was computed by checking how many of these 10000 samples were not contained in the approximated reachable set. 
In only $6$ experiments, the coverage error exceeded $\epsilon=0.02$, confirming that the confidence $1-\delta$ is greater than $99\%$.
%independent experiments for the example shown in \figureref{fig: Calibrated_sublevels_M=1000} (c). Each experiment involved computing an empirical error and empirical confidence by sampling a new calibration set of size $N=200$ and another 10000 samples from the reachable set. The empirical error was computed by checking how many of these 10000 samples were not contained in the approximated reachable set. The fraction of samples that were not contained in the approximated reachable set provided a measure of the empirical error. \theoremref{thm:calibrated_sublevels} guarantees that $99\%$ of the times, the approximation error will be lower than $2\%$. Out of the $1000$ experience, only $6$ gave an approximation error greater than $2\%$, ensuring that the confidence is indeed greater than $99\%$.
\begin{figure}[tbp]
\centering
\floatconts
  {fig:Calibrated_sublevels_M=1000}
  {\caption{Reach set approximations (outlined in purple) from Example \ref{ex:example2}, %This plot presents a similar experiment as shown in \figureref{fig:Calibrated_sublevels_M=10000} 
  with a reduced sample size of $M=1000$, of which $N=200$ are used as a calibration set.}}
  {%
    \subfigure[$d=6$, $\varepsilon=0.02$]{\label{fig:image-a-m1000}%
      \includegraphics[width=0.32\linewidth]{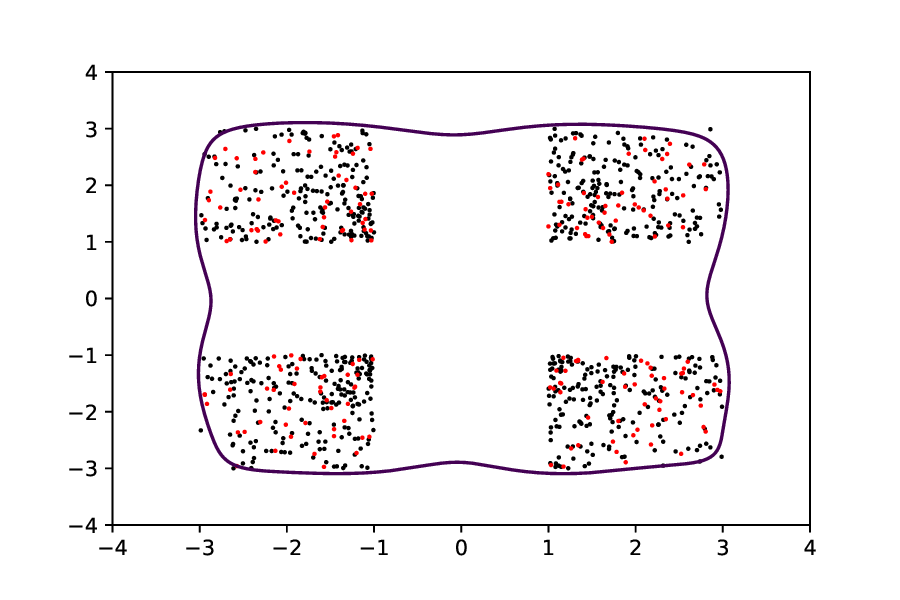}}%
    %\qquad
    \subfigure[$d=10$, $\varepsilon=0.02$]{\label{fig:image-b-m1000}
\includegraphics[width=0.32\linewidth]{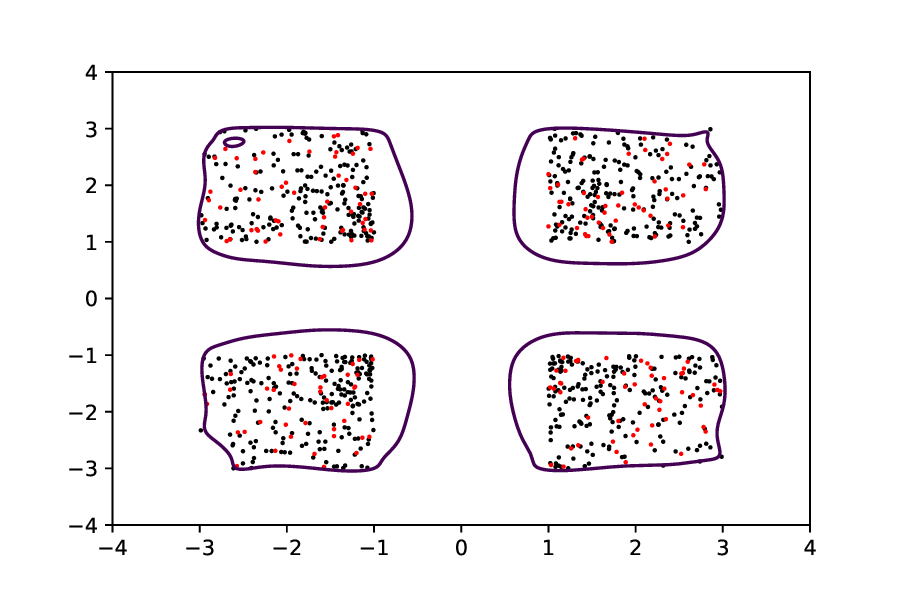}}  %\qquad
    \subfigure[$d=15$, $\varepsilon=0.02$]{\label{fig:image-c-m1000}
\includegraphics[width=0.32\linewidth]{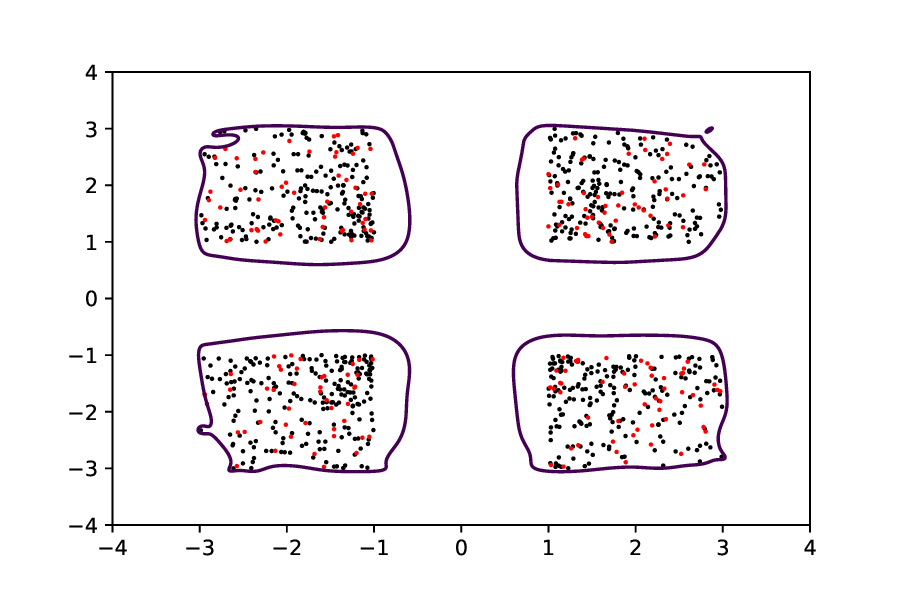}}   % \qquad
  }
\end{figure}
\end{example}
\subsection{Avoiding the Calibration Set}
\label{ssec:Avoiding the Calibration Set}

In this section, we circumvent split between training and calibration sets by using \emph{trans\-duc\-tive conformal prediction} \cite{vovk2013transductive}. Transductive conformal prediction is a method used to construct prediction regions for a new data point without relying on a separate training set or calibration set. The calibration set is taken to be the entire training set plus the point at which the function is evaluated, in other words a new non conformity is modulated by the data point. 
The statistical guarantees of the previous section, and in particular of Theorem\;\ref{thm:calibrated_sublevels}, hold also for this choice of nonconformity function, with $\set{D}_\mathrm{cal} := \set{D}$.
%In sum, 
This approach allows us to use all the available sample points from the measure $\mu$ to train the Christoffel function and compute the conformal region, but at the price of higher computational cost, as will be discussed below.

Let the training set be $\set{D} = \{ \Vec{x}^1 ,\Vec{x}^2,...,\Vec{x}^N  \}$ be $N$ i.i.d samples from the probability distribution $\mu$. To compute the p-value at any point $\Vec{x} \in  \mathbb{R}^n$, we add $\Vec{x}$ to the set $\set{D}$ before computing the empirical Christoffel polynomial. 
 Let $\set{D}_x = \set{D} \cup \{ \Vec{x} \}$, let the empirical measure for $\set{D}_x$ be $\hat{\mu}_x$, and let $\widehat{\mathbf{M}}_x$ be its moment matrix. 
Using $\set{D}_x$ in the empirical Christoffel polynomial, we get the nonconformity function
 $$
r(\Vec{x}) =  \Lambda^{-1}_{\hat{\mu}_x, d}(\Vec{x}) = \mathbf{v}_{d}(\Vec{x})^T \widehat{\mathbf{M}}_x^{-1}\mathbf{v}_{d}(\Vec{x}).
 $$
 %The p-value at $\Vec{x}$ is: \[  p_{value}^x(\Vec{x})=\frac{\left|\bigl\{i  \bigm| \Lambda^{-1}_{\hat{\mu}_x, d}(\Vec{x}^{i})\geq \Lambda^{-1}_{\hat{\mu}_x, d}(\Vec{x})\bigr\}\right|}{N} \]
%
 We now have to evaluate a different empirical Christoffel polynomial each time we evaluate the p-value
 \[  p_{value}(\Vec{x})=\tfrac{1}{N}\left|\bigl\{i  \bigm| \Lambda^{-1}_{\hat{\mu}_x, d}(\Vec{x}^{i})\geq \Lambda^{-1}_{\hat{\mu}_x, d}(\Vec{x})\bigr\}\right| \]
%$\Lambda_{\hat{\mu}_x, d}(\Vec{x})$. 
In particular, we need to compute a new moment matrix and invert it for each evaluation. This is computationally expensive, on the order of $\mathcal{O}(s(d)^3)$. To avoid this, we compute the inverse moment matrix of the set $\set{D}_x$ incrementally using the Sherman-Morrison formula, as proposed by \cite{ducharlet2022leveraging}. 
This allows us to replace the evaluation of $\Lambda^{-1}_{\hat{\mu}_x, d}(\Vec{x})$, which depends on $\Vec{x}$, with evaluations of the original Christoffel polynomial $\Lambda^{-1}_{\hat{\mu}, d}(\Vec{x})$, plus one additional product:
\begin{equation}
\Lambda^{-1}_{\hat{\mu}_x, d}(\Vec{x}) 
= \frac{\Lambda^{-1}_{\hat{\mu}, d}(\Vec{x})}{1+\Lambda^{-1}_{\hat{\mu}, d}(\Vec{x})}, \quad
\Lambda^{-1}_{\hat{\mu}_x, d}(\Vec{x}^{i}) = \Lambda^{-1}_{\hat{\mu}, d}(\Vec{x}^{i})
  - \frac{\bigl( \mathbf{v}_{d}(\Vec{x})^\intercal \Vec{y}^{i} \bigr)^2 }{1+\Lambda^{-1}_{\hat{\mu}, d}(\Vec{x})},
\end{equation}
where $\Vec{y}^{i} = \widehat{\mathbf{M}}_d^{-1}\mathbf{v}_{d}(\Vec{x}^{i})$ are vectors that can be precomputed.
The cost of precomputing $\Lambda^{-1}_{\hat{\mu}, d}(\Vec{x}^{i})$ and the vectors $\Vec{y}^{i}$ is $\mathcal{O}(Ns(d)^2)$, with storage requirements $\mathcal{O}(Ns(d))$.
The reduces the cost of evaluating $\Lambda^{-1}_{\hat{\mu}_x, d}(\Vec{x}^{i})$ for a given $\Vec{x}$ to $\mathcal{O}(s(d))$.
The resulting cost of evaluating
$p_{value}(\Vec{x})$ is $\mathcal{O}(Ns(d)+s(d)^2)$.

\begin{figure}[tbp]
\floatconts{fig:transductive}
  {\caption{Reach set approximation of example \ref{ex:foursquares} using the transductive conformal prediction %\algorithmref{alg:transductive} 
  and a Christoffel polynomial of degree $d=15$, which avoids the split into training and calibration sets. %. The training set, depicted by the black dots, was also utilized to compute the conformal region. \theoremref{thm:calibrated_sublevels} assert that the error uncertainty for this figure will be lower than $0.0045$ with confidence $99\%$.
  }}
  {\includegraphics[width=0.4\linewidth]{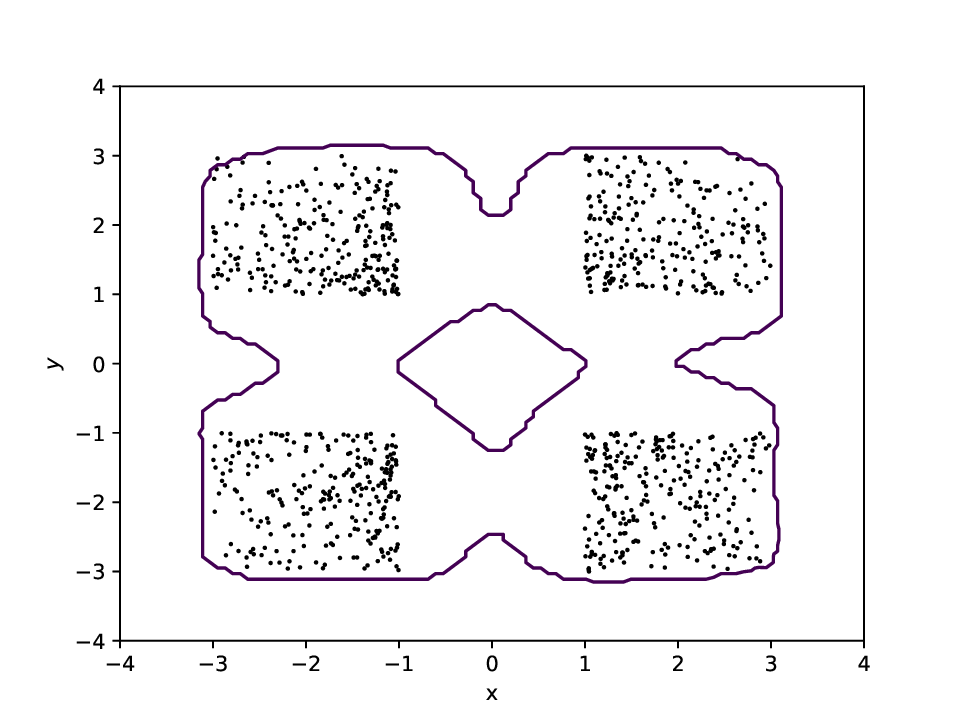}}
\end{figure}
\begin{example}
%\\ To showcase the empirical results obtained from applying the transductive conformal prediction variant of our algorithm, which incorporates an incremental version of the Christoffel function. 
Building on example \ref{ex:foursquares}, \figureref{fig:transductive} shows the reachable set approximation obtained using transductive conformal prediction with a Christoffel function of degree 15. In this case, we use the same $M=N=1000$ sample points to train the Christoffel function and compute the set approximation. The guarantees provided by \theoremref{thm:calibrated_sublevels} assert that, using a training set of 1000 samples, the coverage error is below $0.45\%$ with confidence $1-\delta = 0.99$.
\end{example}

\section{Robustness to Outliers}
\label{sec:Robustness to Outliers}
In this section, we address the presence of outliers in the data set. As data may not be very abundant in real-life applications, one may have to work with a calibration set containing outliers without knowing which data point is an outlier and which one isn't. The presence of outliers in the training set does not affect the theoretical guarantees obtained using conformal prediction theory, though it will affect the tightness of the approximated reachable set. On the other hand, the presence of outliers in the calibration set will impact those guarantees. 

The following theorem provides PAC guarantees on the reach set approximation even with outliers in the calibration set.
%, we can still recover a $\epsilon$-accurate approximated reachable set. Note that we do not know which samples in the calibration set are outliers and which are inliers. 
Under the assumption that no more than $p$ outliers are in the calibration set $\set{D}$, the confidence in the result depends on $\epsilon$, $p$, and the size of the calibration set $N$. 
%
%\section{Christoffel Functions for Conformal Prediction?????}
%
\begin{theorem}\label{thm:conformal outliers}
Consider a set of points $\set{D} = \{ \Vec{x}^1 ,\Vec{x}^2,...,\Vec{x}^N  \}$ containing no more than $p$ outliers, with $2p +1 < N$, and where the rest of samples are i.i.d from a probability measure $\mu$. Then for any i.i.d vector $\Vec{x}$ sampled from $\mu$ and $\epsilon \in (0,1)$,
\begin{equation}\label{eq:outliers}
    \mathbb{P}\biggl( \mu\Bigl(C_{\set{D}}^{\frac{p+1}{N}}\Bigr)  \geq 1-\epsilon \biggr) \geq \sum_{i = p+1}^{N-p} \tbinom{N-p}{i} \epsilon^i (1-\epsilon)^{N-p-i}
\end{equation}
\end{theorem}
%
%Under the absence of outliers, we have \eqref{eq:final1}. 
This bound is tight in the sense that for $p=0$, \eqref{eq:outliers} is identical to the case without outliers, i.e., we obtain \eqref{eq:final1}.
\begin{proof}
     %Let $\set{D}=\left(\Vec{x}^{(i)}\right), i=1, \ldots, N $ be the calibration set and assume that a maximum number $p$ of those samples are outliers with $2p +1 < N$. 
     Let $\set{D} = \set{D}_{inlier} \cup \set{D}_{oulier}$, with $m \leq p$ being the unknown real size of $\set{D}_{outlier}$.
     Let $U_1, \ldots, U_{N-m} \stackrel{\text { i.i.d. }}{\sim} \operatorname{Unif}([0,1])$, with order statistics $U_{(1)} \leq U_{(2)} \leq \ldots \leq U_{(N-m)}$.
     \\
     For $\epsilon \in (0,1)$ let $b_1 = ...= b_{p+1} = \epsilon $ and $ b_{p+2} = ...=b_{N-p}=... = b_{N-m} = 1 $. Then %we have 
     $\forall m \leq p$:
\begin{align*}
\mathbb{P}\left[U_{(1)} \leq b_1, \ldots, U_{(N-m)} \leq b_{N-m}\right] 
& \geq \mathbb{P}\left[U_{(1)} \leq b_1, \ldots, U_{(N-p)} \leq b_{N-p}\right] \\
& \geq \sum_{i = p+1}^{N-p} \tbinom{N-p}{i} \epsilon^i (1-\epsilon)^{N-p-i}.
\end{align*}
The above result is obtained by the following reasoning: let $ 0 < i \leq N-p $, if we have $N-p$ random variable $V_i, \ldots, V_{N-p} \stackrel{\text { i.i.d. }}{\sim} \operatorname{Unif}([0,1])$ the probability to have exactly $i$ of them below $\epsilon$ is equal to $\binom{N-p}{i}\epsilon^i (1-\epsilon)^{N-p-i}$, therefore, the probability of having at least $p+1$ of them below $\epsilon$ is equal to $\sum_{i = p+1}^{N-p} \binom{N-p}{i} \epsilon^i (1-\epsilon)^{N-p-i}.$
\\Let $\Vec{x}$ be an i.i.d vector sampled from $\mu$. By definition, $$ \mu \Bigl(   C_{\set{D}_{inliers}}^{\frac{p+1}{N-m}}   \Bigr) = \mathbb{P}\Bigl( \Vec{x} \in C_{\set{D}_{inliers}}^{\frac{p+1}{N-m}} \Bigm| \set{D}_{inliers}\Bigr).$$ 
Using \theoremref{thm:calibration},  we get : 
\[ \mathbb{P}\biggl( \mu\Bigl(C_{\set{D}_{inliers}}^{\frac{p+1}{N-m}}\Bigr)  \geq 1-\epsilon \biggr) \geq \sum_{i = p+1}^{N-p} \tbinom{N-p}{i} \epsilon^i (1-\epsilon)^{N-p-i} \]
Since $ C_{\set{D}_{inliers}}^{\frac{p+1}{N-m}} \subseteq C_{\set{D}}^{\frac{p+1}{N}} $, we have $
     %\begin{equation}\label{eq:outliers2}
     \mu\Bigl( C_{\set{D}}^{\frac{p+1}{N}}\Bigr) \geq \mu\Bigl( C_{\set{D}_{inliers}}^{\frac{p+1}{N-m}}\Bigr),
     $
    % \end{equation}
which leads us %with \eqref{eq:outliers2}
to \eqref{eq:outliers}. 
%
%\[ \mathbb{P}\left[ \mu(C_{\set{D}}^{\frac{p+1}{N}})  \geq 1-\epsilon \right] \geq \sum_{i = p+1}^{N-p} \tbinom{N-p}{i} \epsilon^i (1-\epsilon)^{N-p-i} \]
\end{proof}
Note that the bound in \theoremref{thm:conformal outliers} \eqref{eq:outliers} is tight in the sense that for $p=0$ we obtain the same lower bound as in \theoremref{thm:calibrated_sublevels} \eqref{eq:final1}.
\begin{table}[tbp]\label{tab:table1}
\floatconts
  {tab:example-hline1}
  {\caption{The confidence bound of \eqref{eq:outliers} for different sizes $N$ of the calibration set and the desired coverage error $\epsilon$ for a calibration set with $5 \%$ outliers or less}}%
  {%
\begin{tabular}{@{}rSSSS@{}}
\toprule
% $\sfrac{\text{size} $N$}{\text{approx rate}}
& \multicolumn{4}{c}{confidence in $\%$} \\
\cmidrule(l){2-5}
size $N$ & \mc{$\epsilon = 4\%$} & \mc{$\epsilon = 5\%$} & \mc{$\epsilon = 6\%$} & \mc{$\epsilon = 10\%$}\\
\midrule
100 & 33  &  51&  68&  96\\
500 &  10 &  42& 77 & 99.99\\
1000 &  3&  37 &  84& 99.99\\
2000 &  0.4 &  31 &  92 & 99.99 \\
\bottomrule
\end{tabular}
  }
\end{table}
\tableref{tab:table1} shows the confidence bound of \eqref{eq:outliers} for different values of the calibration set size and the approximation uncertainty $\epsilon$ under the assumption that no more than $5 \%$ of the calibration set are outliers. 
%The values indicate that to get high confidence in our approximation, the approximation uncertainty has to be bigger than the maximum outlier rate in the calibration set.
We observe that the confidence rapidly approaches $100\%$ when the admissible coverage error is above the ratio of outliers; it rapidly drops to $0\%$ when it is below.

\begin{example}
%\\ For testing \algorithmref{alg:Reachability analysis with outliers}, recall that the presence of outliers in the training set does not affect the theoretical guarantees provided by Theorems \ref{thm:calibrated_sublevels} and \ref{thm:conformal outliers}, as long as the calibration set is independent of the non-conformity function. To account for outliers, we generated $M$ samples from the reachable set and $p$ i.i.d. samples outside the reachable set and combined them. We then used $M-N$ samples from the concatenated set as the training set and $N$ samples as the calibration set. 
To evaluate the performance of \algorithmref{alg:Reachability analysis with outliers} on example \ref{ex:foursquares}, we construct a data set from $M=1500$ samples of the reach set and substitute $10\%$ with outliers, i.e., i.i.d. samples outside the reachable set. We use a calibration set of size $N=500$,  and the rest of the samples are used as a training set to compute the empirical Christoffel polynomial. \figureref{fig:algo3} shows the resulting approximation. With \theoremref{thm:conformal outliers}, the coverage error $\epsilon = 0.15$ with a confidence = $98.9\%$. 
\begin{figure}[tbp]
\floatconts{fig:algo3}
  {\caption{An approximation of the reach set of example \ref{ex:foursquares} (purple outline) obtained with \algorithmref{alg:Reachability analysis with outliers} using a Christoffel polynomial of degree 15, on a data set with $10\%$ outliers. The training set is shown in black, the calibration set in red. }}  {\includegraphics[width=0.4\linewidth]{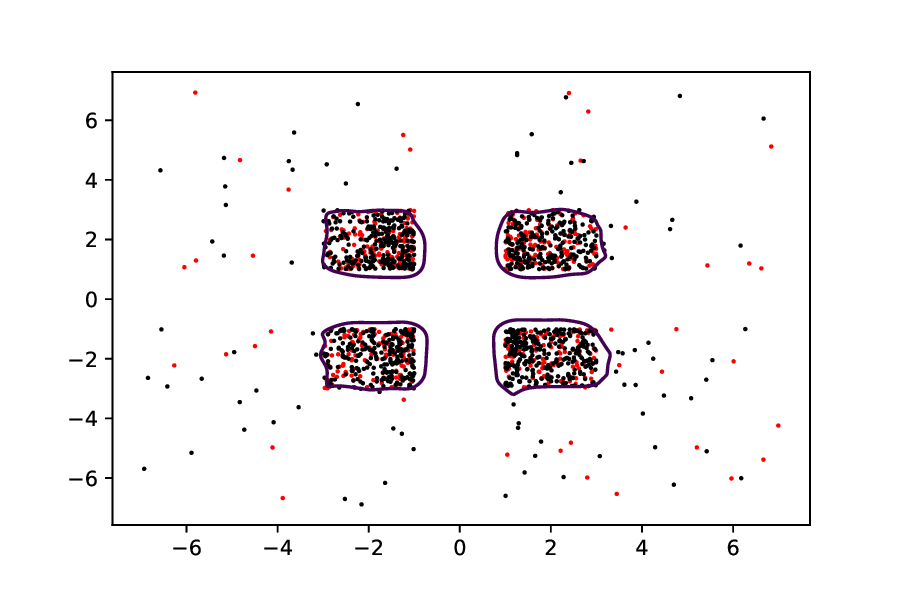}}
\end{figure}
To empirically confirm these bounds, as in Example\;\ref{ex:example2}, we repeat the experiment 1000 times with different samples. For each experiment, we take 10000 samples of the reach set in order to compute the empirical coverage error. None of the experiments resulted in an empirical coverage error above $15\%$, which is consistent with the theoretical guarantee of $98.9\%$ confidence.
\end{example}

\begin{algorithm2e}
\caption{Reachability analysis with outliers}
\label{alg:Reachability analysis with outliers}

\KwIn{Transition function $f$; % random variables 
initial set $\set{I} \subset \mathbb{R}^n$; Christoffel function order $d$, $N$ the size of the calibration set and $p$ the upper bound number of outliers in the calibration; $M$ the total number of simulations with $M > N$ and an i.i.d data sample 
$\set{D} = \{\Vec{x}^{(i)}$ \} for $i \in \{1,...,M\}$. The sample $\Vec{x^{(i)}}$ is an \emph{inlier} if $\Vec{x^{(i)}} \in f(\set{I})$ and an \emph{outlier} otherwise.  . }
\KwOut{Set $\hat{\set{S}}$ representing an $\epsilon$-accurate approximation of the true reachable set $\set{S}$ with confidence $\sum_{i = p+1}^{N-p} \binom{N-p}{i} \epsilon^i (1-\epsilon)^{N-p-i}$}

\begin{enumerate}
  \item Compute the empirical moment matrix with the associated Christoffel degree $d$ using $M-N$ samples:
   $ \widehat{\mathbf{M}}_d =\frac{1}{M-N} \sum_{i=N+1}^{M} \mathbf{v}_{d}\left(\Vec{x^{(i)}}\right) \mathbf{v}_{d}\left(\Vec{x^{(i)}}\right)^{\top} $
     \item Use a calibration set of $N$ samples: $\set{D}_\mathrm{cal} = \{ \Vec{x^{(i)}} \mid i \in \{1,..,N\}\}$ and:
      \begin{enumerate}
          \item Compute the  scores: %\left( \mathbf{v}_{d}\left(\Vec{x^{(1)}}\right)^{\top} {\widehat{\mathbf{M}}}^{-1}_d  \mathbf{v}_{d}\left(\Vec{x^{(1)}}\right),...,\mathbf{v}_{d}\left(\Vec{x^{(N)}}\right)^{\top} {\widehat{\mathbf{M}}}^{-1}_d \mathbf{v}_{d}\left(\Vec{x^{(N)}}\right) \right)$
          $\text{score}_i = \mathbf{v}_{d}(\Vec{x^{(i)}})^{\top} {\widehat{\mathbf{M}}}^{-1}_d  \mathbf{v}_{d}(\Vec{x^{(i)}})$ for $i=1,\ldots,N$
          \item Sort the scores in descending order such that: $\text{score}_1 \geq \text{score}_2 \geq ... \geq \text{score}_N $
          \item  Set the conformal region $C_{\set{D}}^{\frac{p+1}{N}}$ as $ \hat{\set{S}}  = \left\{\Vec{x} \in \mathbb{R}^{n}: \mathbf{v}_{d}(\Vec{x})^{\top} {\widehat{\mathbf{M}}}^{-1}_d  \mathbf{v}_{d}(\Vec{x}) \leq \text{score}_{p+1}  \right\}   $
      \end{enumerate}

\end{enumerate}

\end{algorithm2e}

\section{Experiments}
\label{sec:Experiments}
%In the experiments already conducted, we evaluated the performance of the algorithms and validated the theoretical guarantees provided by the theorems. 
We now turn our focus to the suitability of the empirical Christoffel polynomial as a non-conformity function.
\subsection{Empirical False Positive Rate}
We start by examining the tightness of the reachable set approximation in example \ref{ex:example2} through the empirical measurement of false positives. 
%In previous tests, we only conducted experiments to measure error uncertainty, which corresponds to the number of false negatives.
%
%
\begin{table}[tbp]
\floatconts
  {tab:example-hline3}
  {\caption{\label{tab:table3}Experimentally estimated false-positive rates for different algorithms applied to the reach set approximation of Example\;\ref{ex:foursquares}, with confidence $1-\delta=99\%$}}%
  {\small
\begin{tabular}{@{}lrrrSS@{}}
\toprule
% $\sfrac{\text{size} $N$}{\text{approx rate}}
Nonconformity function  & $|\set{D}|$ & $|\set{D}_\mathrm{train}|$ & $|\set{D}_\mathrm{cal}|$ & \mc{$\epsilon$ in \%} & \mc{FP$\%$} \\
\midrule
Christoffel with $d=6$& 10000 & 8000 & 2000 & 0.2 & 49.5\\
Christoffel with $d=10$&  &  &  & 0.2 & 39.5\\
Christoffel with $d=15$&  \multirow{3}{*}{\vdots} & \multirow{3}{*}{\vdots} & \multirow{3}{*}{\vdots} & 0.2 & 11.7\\
Christoffel with $d=18$&  \multirow{3}{*}{\vdots} & \multirow{3}{*}{\vdots} & \multirow{3}{*}{\vdots} & 0.2 & 7.2\\
LOF score &  &  &  & 0.2 & 3.4 \\
IsolationForest score &  &  &  & 0.2 & 92.9 \\
Oneclass SVM score &  &  &  & 0.2 & 65.7 \\
Christoffel with $d=6$& 1000 & 800 & 200 & 2.2 & 44.6\\
Christoffel with $d=10$&  &  &  & 2.2 & 20\\
Christoffel with $d=15$& \multirow{3}{*}{\vdots} & \multirow{3}{*}{\vdots} & \multirow{3}{*}{\vdots} & 2.2 & 12.7\\
Christoffel with $d=18$& \multirow{3}{*}{\vdots} & \multirow{3}{*}{\vdots} & \multirow{3}{*}{\vdots} & 2.2 & 12.4\\
LOF score &  &  &  & 2.2 & 10.6 \\
IsolationForest score &  &  &  & 2.2 & 86.8 \\
Oneclass SVM score &  &  &  & 2.2 & 60.7 \\
Transduct. Christ. with $d=15$ & 1000 & 1000 & 1000 & 0.5 & 46.6\\
\bottomrule
\end{tabular}
\vspace{1ex}

{ \footnotesize{\raggedright 
$\epsilon$ = Coverage error, at least $1-\epsilon$ of the measure is covered; 
FP$\%$ = False positives in \%, measured by uniform sampling of a sufficiently large bounding box and counting samples in $\hat S \setminus S$}}
  }
\end{table}

%In the following analysis, w
We compare the empirical Christoffel polynomial with other prevalent nonconformity functions: one-class SVM, Isolation Forest \citep{isolation}, and Local Outlier Factor (LOF), as shown in \figureref{fig:svm-isolation-lof}. 
%The application of these algorithms in conformal prediction offers similar theoretical guarantees as the empirical Christoffel polynomial. Among these alternatives,
Only the approximation using LOF seems comparable to that of the Christoffel polynomial, while Isolation Forest exhibits significant variability depending on the random seed.

To gauge the number of false positives and assess the accuracy of the reachable set approximation, we generated 10,000 uniformly distributed samples within the domain $[-4, 4]^2$. The false-positive rate was empirically determined for various degrees $d$, as shown in \tableref{tab:table3}. As observed in earlier plots, a higher degree results in a more accurate fit of the reachable set. The false-positive rates for the other algorithms can also be observed in \tableref{tab:table3} for varying sizes of the training and calibration sets. Consistent with the findings from the \figureref{fig:svm-isolation-lof}, only the LOF provides results that are comparable in quality to those obtained using the Christoffel polynomial.

\begin{figure}[!t]%[tbp]
\floatconts
  {fig:svm-isolation-lof}
  {\caption{Reach set approximations (purple outline) of Example\;\ref{ex:foursquares} using one-class SVM, Isolation Forest, and Local Outlier Factor (LOF) as nonconformity functions, for a common training set of size 800 (black dots) and calibration set of size 200 (red dots). 
  %The algorithms are implemented with defaul hyperparameters from the scikit-learn library. The plots demonstrate the performance differences among these algorithms in conformal prediction.
  }}
  {%
    \subfigure[One-class SVM]{\label{fig:image-a-svm}%
      \includegraphics[width=0.32\linewidth]{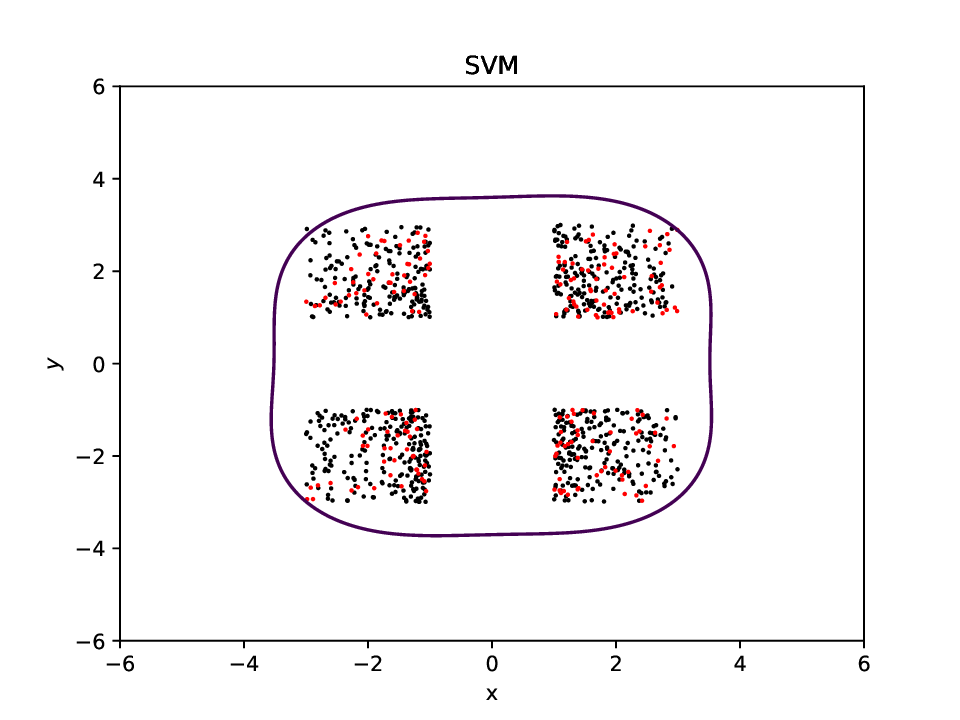}}%
    %\qquad
    \subfigure[Isolation Forest]{\label{fig:image-b-isof}%
      \includegraphics[width=0.32\linewidth]{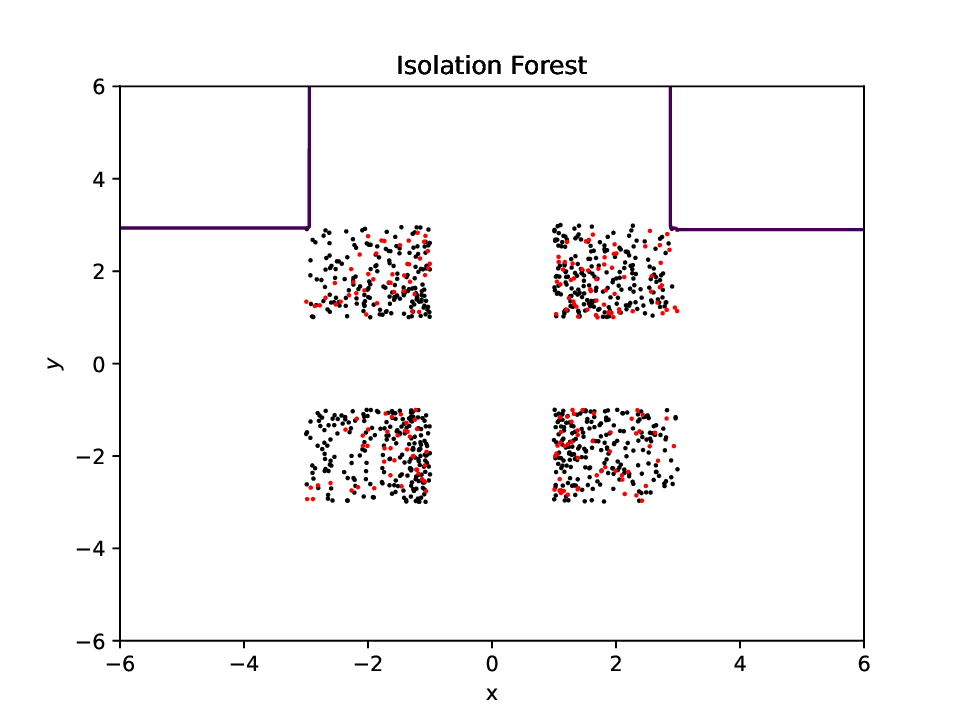}}   %\qquad
    \subfigure[LOF]{\label{fig:image-c-lof}%
      \includegraphics[width=0.32\linewidth]{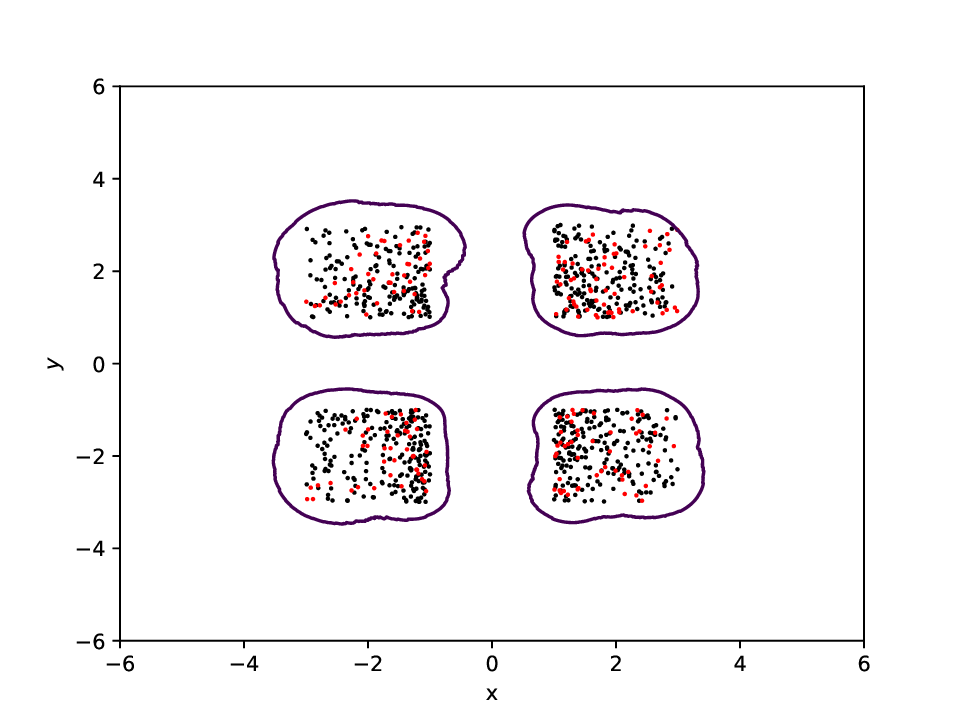}}   % \qquad
  }
\end{figure}
To further demonstrate the effectiveness of the empirical Christoffel polynomial as a non-conformity function, we examine its robustness in the presence of outliers within the training set. Although the theoretical guarantees discussed in this article and in general conformal prediction hold for any choice of non-conformity function, even with outliers in the training set, the presence of these outliers can impact the accuracy of the model. To compare the empirical Christoffel polynomial with LOF, we conducted two experiments.
In the first experiment, we considered the region $[-1,1]^2$ as the reachable set to approximate. We focused on comparing the performance of the algorithms under the presence of outliers in the training set. We generated a training set of size 1,200 containing 200 outliers and a calibration set of size 200, all belonging to the reachable set.
The second experiment was similar to the first one, with a star-shaped region as the reachable set. We generated a training set of size 900 containing 100 outliers and a calibration set of size 200. Figure 9 illustrates how the empirical Christoffel polynomial and LOF approximate the true reachable set in the presence of outliers.

\begin{figure}[!t]%[tbp]
\floatconts
  {fig:small square}
  {\caption{Comparison of reachable set approximations for the empirical Christoffel polynomial (degree 10) and LOF in the first experiment, with the region $[-1,1]^2$ as the target. The training set, containing outliers, is represented by black dots, while the calibration set is shown in red. The plot highlights the performance differences and robustness of both methods in the presence of outliers, demonstrating how the empirical Christoffel polynomial is far more robust.}}
  {%
    \subfigure[Christoffel polynomial]{\label{fig:image-a-svm}%
      \includegraphics[width=0.4\linewidth]{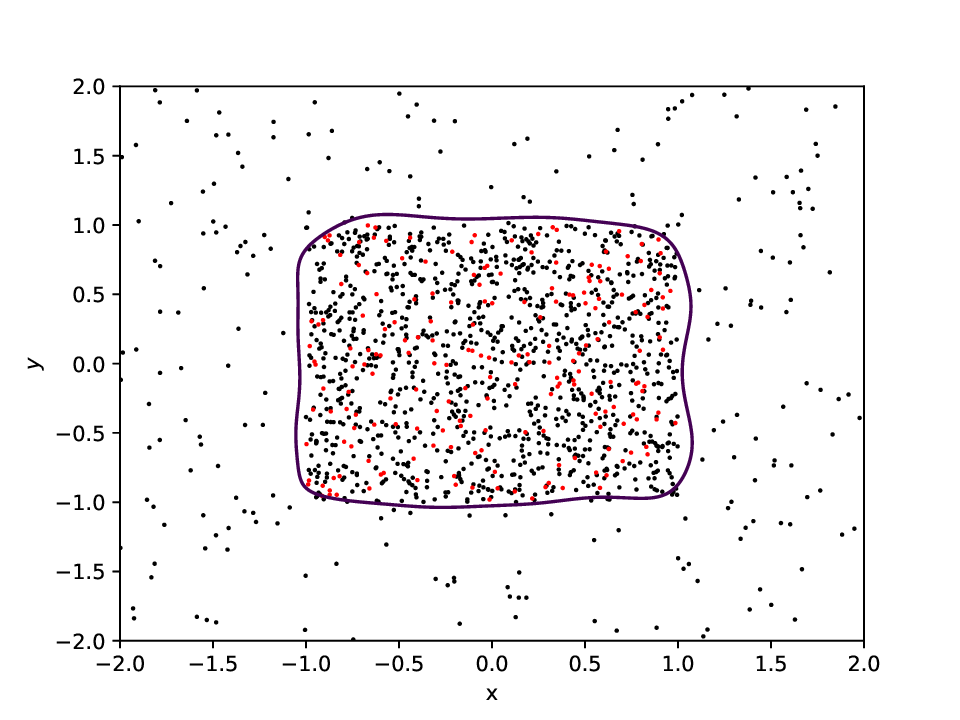}}%
    %\qquad
    \subfigure[LOF]{\label{fig:image-b-isof}%
      \includegraphics[width=0.4\linewidth]{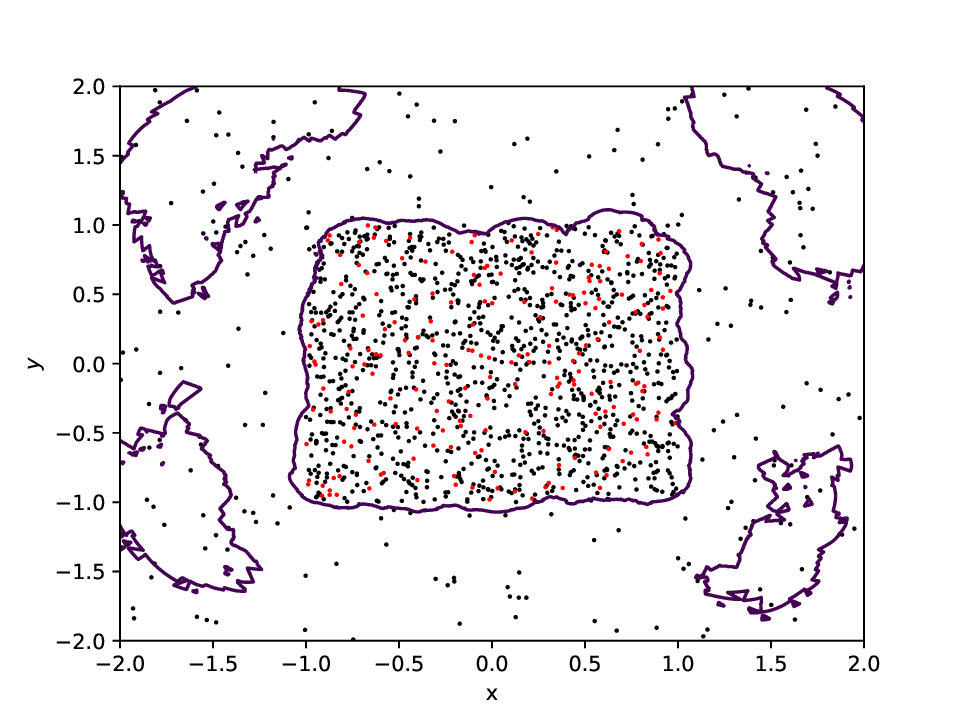}}\\   %\qquad
   
  }
\end{figure}\begin{figure}[!t]%[tbp]
\floatconts
  {fig:star shape}
  {\caption{A comparison of reach set approximation (purple outline) using the Christoffel polynomial with degree 15 and LOF for the second experiment, which targets a star-shaped region. Training set samples are in black and calibration set in red. This plot highlights the performance and robustness of both methods when encountering outliers in a complex geometric scenario, illustrating the effectiveness of the empirical Christoffel polynomial under the presence of outliers.}}
  {%
    \subfigure[Christoffel polynomial]{\label{fig:image-a-svm}%
      \includegraphics[width=0.4\linewidth]{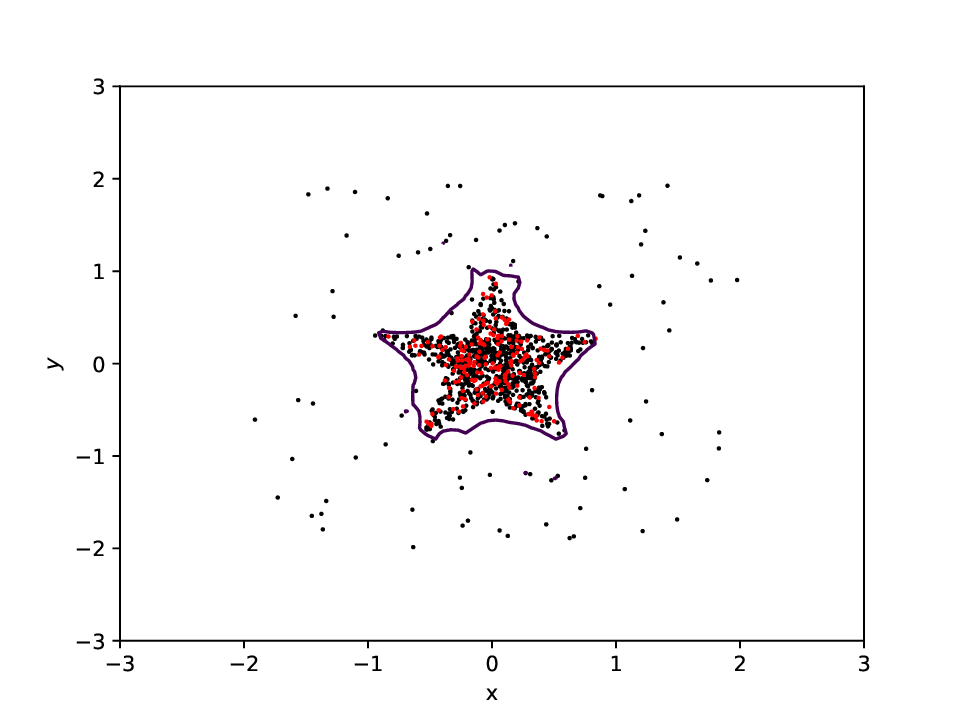}}%
    %\qquad
    \subfigure[LOF]{\label{fig:image-b-isof}%
      \includegraphics[width=0.4\linewidth]{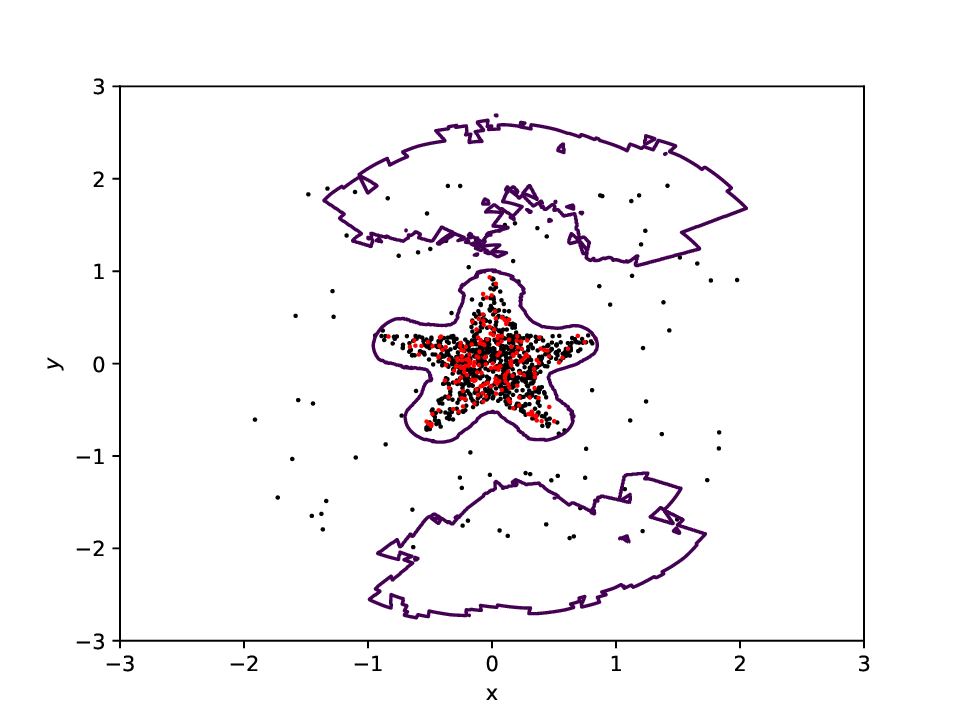}}\\   %\qquad
   
  }
\end{figure}
\figureref{fig:small square,fig:star shape} display the performance of both the empirical Christoffel polynomial and LOF in handling outliers within the training set across distinct and complex geometric situations. When employed as a non-conformity function, the empirical Christoffel polynomial demonstrated greater robustness in the presence of outliers across both experiments.

%\subsection{Experiments on 2D simulations}%{4 squares}
%
\subsection{Duffing oscillator}%{Duffing}
The Duffing oscillator is a nonlinear mathematical model that captures the behavior of a system that oscillates when subject to an external force. It has been used in a variety of physical systems, from mechanical vibrations to biological dynamics. The Duffing oscillator is described by the following nonlinear second-order differential equation: 
\[ {\ddot{x}}= -\delta\dot{x} +\alpha x -\beta x^3 +\gamma cos(\omega t)\]
Similar to \cite{devonport}, we take $\alpha=1$, $\beta= 1,\delta = 0.05, \gamma= 0.4$ and $\omega=1.3$.
We choose the initial set to be $ \set{I}=[-0.95,1.05]\times[-0.05,0.05]$.
\figureref{fig:subfigex} shows an approximation of the reach set, computed with the Christoffel function as nonconformity function for different degrees. 
We observe that for increasing degrees, the approximation is more precise and is able to recover holes. The results are comparable to those reported by \cite{devonport}, where no split into training and calibration sets was carried out.

\begin{figure}[!t]%[tbp]
\floatconts
  {fig:subfigex}
  {\caption{Reach set approximation (purple outline) of the duffing oscillator using the Christoffel polynomial with the data set split into training (black) and calibration set (red), for different degrees $d$ of the Christoffel function, with corresponding coverage error $\varepsilon$ for confidence $1-\delta = 0.99$.}}
  {%
    \subfigure[$d=6$, $\varepsilon=0.002$]{\label{fig:image-a-duff}%
      \includegraphics[width=0.31\linewidth]{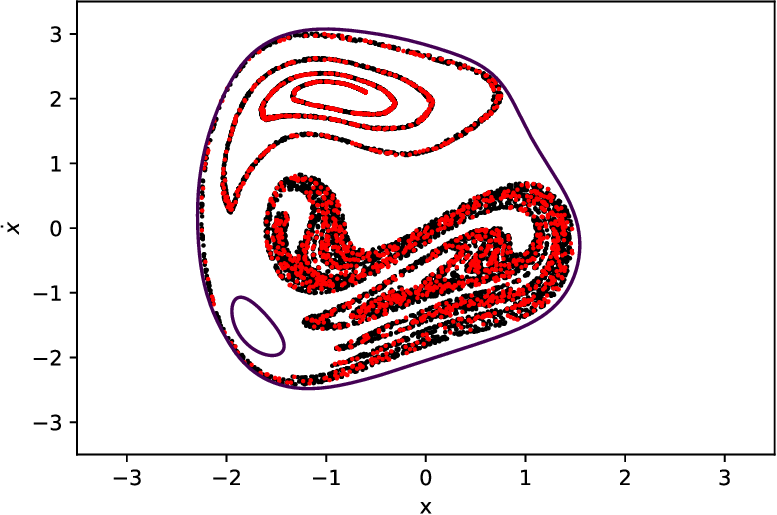}}%
    %\qquad
    \subfigure[$d=10$, $\varepsilon=0.002$]{\label{fig:image-b-duff}%
      \includegraphics[width=0.31\linewidth]{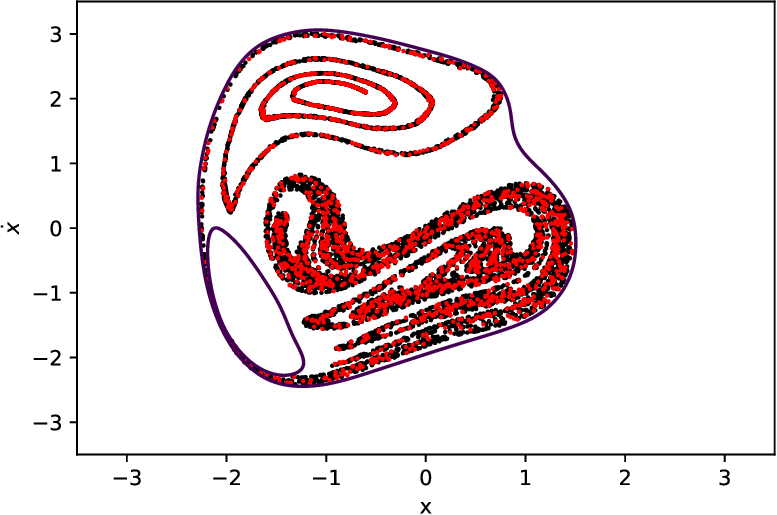}}   %\qquad
    \subfigure[$d=15$, $\varepsilon=0.002$]{\label{fig:image-x-duff}%
      \includegraphics[width=0.31\linewidth]{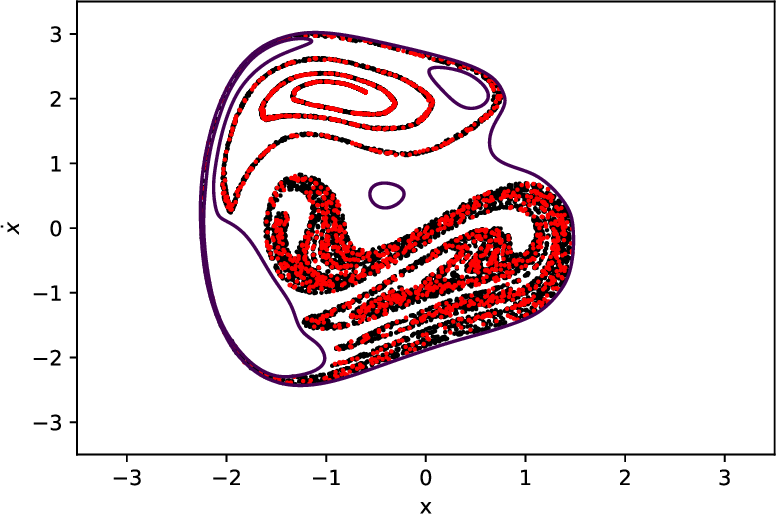}}   % \qquad
  }
\end{figure}

% \begin{figure}[tbp]
% \floatconts
%   {fig:Duffing_algo2}
%   {\caption{}}
%   {\includegraphics[width=0.5\linewidth]{Duffing_Cf_degree_15_after_10periods.eps}}
% \end{figure}

\section{Conclusion}
%
%In this paper, we applied conformal prediction theory to the problem of approximating the one-step reach set of a dynamical system. The reach set is approximated with a particular SOS polynomial, called the Christoffel polynomial, as
%Faïcel : peut être copier coller des bouts de ci-dessous
In this paper, we studied the mathematical reach set approximation in the analysis of dynamical systems based on conformal prediction. We consider for the first time the use of the Christoffel function as a nonconformity function, thanks to its attractive properties in set and density approximation. Our conformal prediction approach provides stronger and more sample-efficient guarantees on reach set approximation and proposed a version of reach set approximation that is robust to outliers,  compared that the most relevant approaches in the literature. We exploited an incremental form of the Christoffel function for transductive conformal prediction that avoids splitting the data into training and calibration sets. 
Extensive illustrative numerical experiments show the effectiveness and the performance of our proposed approach and its associated algorithms. 

The theoretical results that we presented here in the context of reach set approximation are equally valid to approximate compact sets, or the support of probability distributions, in other application domains. Naturally, the computation of the Christoffel function is subject to numerical errors. The impact of such numerical issues will be studied in future work.

\acks{This work has been supported by the French government under the ``France 2030" program as part of the SystemX Technological Research Institute. This work was conducted as part of the Confiance.AI program, which aims to develop innovative solutions for enhancing the reliability and trustworthiness of AI-based systems.}

\bibliography{tebjou-biblio,hybrid}

%\appendix

%\section{First Appendix}\label{apd:first}

%This is the first appendix.

\end{document}